\documentclass[journal]{IEEEtran}
\usepackage{graphicx}
\usepackage{subfigure} 
\usepackage{booktabs,xcolor} 
\usepackage{amsmath,dsfont}
\usepackage{cite,bm,color,url}
 
\usepackage{hyperref}
\usepackage{amsfonts} 
\usepackage{balance}
\usepackage{algorithmicx,algorithm,algpseudocode}

\input{mysymbol.sty}

\def \bbeta {{\boldsymbol \beta}}

\begin{document}

\title{\huge{Communication-Efficient Policy Gradient Methods for Distributed Reinforcement Learning}}
\author{
Tianyi Chen, Kaiqing Zhang, Georgios B. Giannakis, and Tamer Ba\c{s}ar
\thanks {Work in this paper was supported by NSF 1509040, 1508993, and 1711471, and US ARL W911NF-17-2-0196. This work was also supported by the Rensselaer-IBM AI Research Collaboration (http://airc.rpi.edu), part of the IBM AI Horizons Network (http://ibm.biz/AIHorizons).}
\thanks{T. Chen is with the Department of Electrical, Computer, and Systems Engineering,
Rensselaer Polytechnic Institute, Troy, NY, USA. Email: chent18@rpi.edu. 
K. Zhang and T. Ba\c{s}ar are with the Department of Electrical and Computer Engineering, University of Illinois at Urbana-Champaign, Urbana, IL 61801, USA. Emails: \{kzhang66,basar1\}@illinois.edu. G. B. Giannakis are with the Department of Electrical and Computer Engineering and the Digital Technology Center, University of Minnesota, Minneapolis, MN 55455 USA. Email: georgios@umn.edu.
}
}
\markboth{ }{}
\maketitle

\begin{abstract}
This paper deals with distributed policy optimization in reinforcement learning, which involves a central controller and a group of learners. In particular, two typical settings encountered in several applications are considered: \emph{multi-agent reinforcement learning} (RL) and \emph{parallel RL}, where frequent information exchanges between the learners and the controller are required. For many practical distributed systems, however, the overhead caused by these frequent communication exchanges is considerable, and becomes the bottleneck of the overall performance. To address this challenge, a novel policy gradient approach is developed for solving distributed RL. The novel approach adaptively skips the policy gradient communication during iterations, and can reduce the communication overhead without degrading learning performance. It is established analytically that:  i) the novel algorithm has convergence rate identical to that of the plain-vanilla policy gradient; while ii) if the distributed learners are heterogeneous in terms of their reward functions, the number of communication rounds needed to achieve a desirable learning accuracy is markedly reduced. Numerical experiments corroborate the communication reduction attained by the novel algorithm compared to alternatives. 
\end{abstract}


\section{Introduction}
Reinforcement learning (RL) involves a sequential decision-making procedure, where a learner takes (possibly randomized) actions in a stochastic environment over a sequence of time steps, and aims to maximize the long-term cumulative  rewards received from the interacting environment. Generally modeled as a Markov decision process (MDP) \cite{sutton2018}, the sequential decision-making process has been tackled by various RL algorithms, including Q-learning \cite{watkins1992q}, policy gradient (PG)  \cite{sutton2000}, and actor-critic methods  \cite{konda2000actor}. While these popular RL algorithms were originally developed for the \emph{single-learner} task, a number of practical RL tasks such as autonomous driving \cite{shalev2016safe}, robotics \cite{stone2000multiagent},  and video games \cite{tampuu2017multiagent}, involve \emph{multiple learners} operating in a distributed fashion. In this paper, we consider the distributed reinforcement learning problem that covers two general RL settings: \emph{multi-agent collaborative} RL and \emph{parallel} RL. The distributed RL settings we consider include a central controller that coordinates the learning processes of all learners. The learners can be agents in the multi-agent collaborative RL, or, workers in the parallel RL. In the former setting, multiple agents aim to maximize the team-averaged long-term reward via collaboration in a common environment \cite{kar2013cal,wai2018nips,zhang18icml}; while in the latter, multiple parallel machines are used for solving large-scale MDPs with larger computational power and higher data efficiency  \cite{nair2015massively,mnih2016icml}. Similar learning paradigms have also been investigated in distributed \emph{supervised learning} \cite{recht2011,li2014scaling}, e.g., {Federated Learning} \cite{mcmahan2017blog,mcmahan2017}. 

To coordinate the distributed learners, the central controller must exchange information with all learners, by collecting their rewards and local observations, or, broadcasting the RL policy to them. This type of information exchange requires frequent communication between the controller and the learners. However, in many applications, including cloud-edge AI systems \cite{stoica2017}, autonomous driving \cite{shalev2016safe}, and other applications in IoT \cite{chen2019iot}, the communication is costly and the latency caused by frequent communication becomes the bottleneck of the overall performance. These considerations motivate well the development of communication-efficient approaches for latency-sensitive distributed RL tasks. Although there has been a surging interest in studying communication-efficient approaches for supervised learning \cite{alistarh2017nips,jordan2018,chen2018nips}, no prior work has particularly focused on the RL setting. In this context, our goal is to develop a simple yet general algorithm for solving distributed RL problems, with provable convergence guarantees and reduced communication overhead.  

\subsection{Our contributions}
Targeting a communication-efficient solver for distributed RL, we propose a new PG method that we term Lazily Aggregated Policy Gradient (LAPG). With judiciously designed communication trigger rules, LAPG is shown capable of: i) achieving the same order of convergence rate
(thus iteration complexity) as vanilla PG under standard conditions; and, ii) reducing the communication rounds required to achieve a  desirable learning accuracy, when
the distributed learners are heterogeneous (meaning reward functions and initial states are not homogeneous). In certain learning settings, we show that LAPG requires only ${\cal O}(1/M)$ communication of PG with $M$ denoting the number of learners. 
Empirically, we evaluate the performance of LAPG using neural network-parameterized policies on a  popular multi-agent collaborative RL benchmark, the cooperative navigation task, and corroborate that LAPG can considerably reduce the communication required by PG.

\subsection{Related work}
\textbf{PG methods.}
PG methods have been recognized as one of the most pervasive RL algorithms \cite{sutton2018}, especially for RL tasks with large and possibly continuous state-action spaces. By parameterizing the infinite-dimensional policy with finite-dimensional vectors \cite{sutton2018}, PG methods reduce the search for the optimal policy over functional spaces to that over parameter spaces. Early PG methods include the well-known REINFORCE algorithm \cite{williams1992}, as well as the variance-reduced G(PO)MDP algorithm \cite{baxter2001jair}. Both REINFORCE and G(PO)MDP are Monte-Carlo sampling-type algorithms that estimate the policy gradient using the rollout trajectory data. To further reduce the variance, a policy gradient estimate that utilizes Q-function approximation was developed in \cite{sutton2000}, based on a policy gradient theorem derived therein. Since then, several PG variants have made significant progress in accelerating convergence  \cite{kakade2002natural}, reducing variance \cite{papini18icml}, dealing with nonlinear policies \cite{paternain2020stochastic}, handling continuous actions \cite{silver2014dpg}, and ensuring policy improvement \cite{papini2017nips}, especially with deep neural networks as function approximators \cite{schulman2015icml,lillicrap2015,schulman2017proximal}. However, all these algorithms were developed for the single-learner setting. 

\textbf{Distributed RL.}
Distributed RL has been investigated in the regimes of both multi-agent RL and parallel RL. The studies of multi-agent RL can be traced back to \cite{claus1998dynamics} and \cite{wolpert1999general}, with applications to network routing \cite{boyan1994packet} and power network control \cite{schneider1999distributed}. All these works, however, rather heuristically build on the direct modification of Q-learning from a single- to multi-agent settings, without performance guarantees.  The first distributed RL algorithm with convergence guarantees has been reported in \cite{lauer2000algorithm}, although tailored for the tabular multi-agent MDP setting. More recently, \cite{kar2013cal} developed a distributed Q-learning algorithm, termed \emph{QD-learning}, over networked agents that can only communicate with their neighbors. In the same setup, fully decentralized actor-critic algorithms with function approximation were developed in \cite{zhang18icml} to handle large or even continuous state-action spaces. From an empirical viewpoint, a number of deep multi-agent collaborative RL algorithms have also been developed \cite{gupta2017cooperative,lowe2017nips, omidshafiei2017deep}.  
On the other hand, parallel RL, which can efficiently tackle the single-learner yet large-scale RL problem by exploiting parallel computation, has also drawn increasing attention in recent years. In particular, \cite{li2011mapreduce} applied the Map Reduce framework to parallelize batch RL methods, while \cite{nair2015massively} introduced the first massively distributed framework for RL. In \cite{mnih2016icml}, asynchronous RL algorithms have also been introduced to solve large-scale MDPs. This parallelism was shown to stabilize the training process, and also benefit data efficiency \cite{mnih2016icml}. Nonetheless, none of them has tackled the efficiency of communication.  

\textbf{Communication-efficient learning.}
Improving communication efficiency in generic distributed learning settings has attracted much attention recently, especially for supervised learning \cite{jordan2018,mcmahan2017blog}. With their undisputed performance granted, available communication-efficient methods do not directly apply to distributed RL, because they are either non-stochastic \cite{chen2018nips,zhang2015icml}, or tailored for convex problems only \cite{stich2018local}. 
However, PG-type methods in distributed RL are inherently dealing with nonconvex stochastic problems. 
Algorithms for nonconvex problems are available e.g., \cite{alistarh2017nips}, but they are designed to minimize the required bandwidth per communication, not the rounds. 
Communication-efficient optimization has also been studied under the name event-triggered control, primarily for the consensus-based control tasks \cite{dimarogonas2011distributed}. 
In recent years, researchers have extensively studied its application to distributed continuous-time optimization in networked systems, e.g., \cite{nowzari2019event,liu2017asynchronous}. However, there are two fundamental differences between the event-triggered control algorithms and our lazy aggregation algorithm. First, the  event-triggered condition is mostly nonadaptive compared to the adaptive communication condition. Second, the goal of event-triggered control is to maximize the time between two consecutive control actions while still ensuring convergence rather than improving the communication complexity.

Aiming to reduce the number of communication rounds,  the lazily aggregated gradient (LAG) algorithm for communication-efficient distributed supervised learning has been developed recently in \cite{chen2018nips}. In addition to solving a different problem, the technical novelty of LAPG relative to LAG in \cite{chen2018nips} lies in i) LAPG applies to the stochastic settings with possibly biased stochastic gradients; and, ii) to overcome the bias and variance, a modified communication rule is developed along with more involved probabilistic arguments. 
Another challenge of developing optimization algorithms for RL relative to existing stochastic settings is that the distribution used to sample data is a function of the time-varying parameters, which introduces non-stationarity. 

Therefore, communication-efficient distributed RL is a challenging task, and so far it has been a less explored territory.

\vspace{0.1cm}
\noindent\textbf{Notation}. Bold lowercase letters denote column vectors, which are transposed by $(\cdot)^{\top}$. And $\|\mathbf{x}\|$ denotes the $\ell_2$-norm of $\mathbf{x}$. Inequalities for vectors $\mathbf{x} > \mathbf{0}$ will be defined entrywise. Symbol $\mathbb{E}$ denotes expectation, $\mathbb{P}$ stands for probability, and $\Delta({\cal S})$ denotes a distribution over $\cal S$.

\section{Distributed Reinforcement Learning}
In this section, we present the essential background on distributed RL and the plain-vanilla PG methods that can be applied to solve the distributed RL tasks.

\subsection{Problem statement}
Consider a central controller, and a group of $M$ distributed \emph{learners}, belonging to a set ${\cal M}:=\{1,\ldots,M\}$.  Depending on the specific distributed RL setting to be introduced shortly, a learner can be either an agent in the multi-agent collaborative RL, or, a worker in the parallel RL setup. As in conventional RL, the distributed RL task can be cast under the umbrella of MDP, described by the following sextuple 
\begin{equation}\label{eq.tuple}
\left({\cal S}, {\cal A}, {\cal P}, \gamma, \rho, \{\ell_m\}_{m\in{\cal M}} \right)  	
\end{equation}
where ${\cal S}$ and  ${\cal A}$ are, respectively, the state space and the action space for all learners; ${\cal P}$ is the space of the state transition kernels defined as mappings ${\cal S}\times {\cal A}\rightarrow \Delta({\cal S})$; $\gamma\in (0,1)$ is the discounting factor; $\rho$ is the initial state distribution; and $\ell_m\!:\!{\cal S}\times {\cal A}\!\rightarrow\! \mathbb{R}$ is the loss (or negative reward) for learner $m$. 

Based on the sextuple \eqref{eq.tuple}, a \emph{policy} that generates a sequence of decisions provides a solution to MDP. We consider the stochastic policy $\bbpi:{\cal S}\rightarrow \Delta({\cal A})$ that specifies a conditional distribution of all possible joint actions given the current state $\mathbf{s}$, where the probability density of taking the joint action $\mathbf{a}$ at a state $\mathbf{s}$ is denoted by $\bbpi(\mathbf{a}|\mathbf{s})$. The commonly used Gaussian policy \cite{deisenroth2010} is a function of the state-dependent mean $\bbmu(\mathbf{s})$ and covariance matrix $\mathbf{\Sigma}(\mathbf{s})$, given by $\bbpi(\,\cdot\,|\mathbf{s})=\mathcal{N}(\bbmu(\mathbf{s}),\mathbf{\Sigma}(\mathbf{s}))$. 
Considering discrete time $t\in\mathbb{N}$ in an infinite-horizon setting, a policy $\bbpi$ can generate a trajectory of state-action pairs ${\cal T}:=\{\mathbf{s}_0,\mathbf{a}_0, \mathbf{s}_1,\mathbf{a}_1, \mathbf{s}_2,\mathbf{a}_2,\ldots\}$ with $\mathbf{s}_t\in{\cal S}$ and $\mathbf{a}_t\in{\cal A}$. 
In distributed RL, the objective is to find the optimal policy $\bbpi$ that minimizes the infinite-horizon discounted loss aggregated over all learners, that is
\begin{align}\label{opt0}
\small
	\min_{\bbpi}\!\sum_{m\in{\cal M}}\!{\cal L}_m(\bbpi)~~{\rm with}~~{\cal L}_m(\bbpi)\!:=\mathbb{E}_{{\cal T}\sim\mathbb{P}(\cdot|\bbpi)}\!\!\left[\sum_{t=0}^{\infty}\gamma^t\ell_m(\mathbf{s}_t,\mathbf{a}_t)\!\right]
\end{align}
where $\ell_m(\mathbf{s}_t,\mathbf{a}_t)$ and ${\cal L}_m(\bbpi)$ are respectively, the loss given the state-action pair $(\mathbf{s}_t,\mathbf{a}_t)$ and the cumulative loss for learner $m$. The expectation in \eqref{opt0} is taken over the random trajectory ${\cal T}$. Given a policy $\bbpi$, the probability of generating trajectory ${\cal T}$ is given by $\mathbb{P}({\cal T}|\bbpi)=\mathbb{P}(\mathbf{s}_0,\mathbf{a}_0, \mathbf{s}_1,\mathbf{a}_1, \mathbf{s}_2,\mathbf{a}_2,\cdots|\bbpi)=\rho(\mathbf{s}_0) \prod_{t=0}^{\infty} \bbpi(\mathbf{a}_t|\mathbf{s}_t)\mathbb{P}(\mathbf{s}_{t+1}|\mathbf{s}_t,\mathbf{a}_t)$, 
where $\rho(\mathbf{s}_0)$ is the probability of the initial state being $\mathbf{s}_0$, and $\mathbb{P}(\mathbf{s}_{t+1}|\mathbf{s}_t,\mathbf{a}_t)$ is the transition probability from the current state $\mathbf{s}_t$ to the next state $\mathbf{s}_{t+1}$ by taking action $\mathbf{a}_t$. Clearly, the trajectory ${\cal T}$ is determined by the underlying MDP and the policy $\bbpi$.

Depending on how different learners   are coupled with each other, the generic distributed RL formulation \eqref{opt0} includes the two popular RL settings, as highlighted next.

\textbf{Multi-agent collaborative reinforcement learning.}
A number of RL applications involve interaction between multiple heterogeneous but collaborative learners (a.k.a. agents), such as those in controlling unmanned aerial vehicle \cite{ponda2015}, and autonomous driving \cite{shalev2016safe}. 
This is referred as the multi-agent collaborative RL. 
The multi-agent RL problem can be modeled as an MDP characterized by the sextuple $\big({\cal S}, \prod_{m\in{\cal M}}\!{\cal A}_m, {\cal P}, \gamma, \rho, \{\ell_m\}_{m\in{\cal M}} \big)$, where each agent $m$ observes a global state $\mathbf{s}_t\in{\cal S}$ shared by all the agents, and takes an action $\mathbf{a}_{m,t}\in {\cal A}_m$ with the local action space denoted as ${\cal A}_m$. 
The local action of agent $m$ is generated by a local policy $\bbpi_m:{\cal S}\rightarrow \Delta({\cal A}_m)$. 
Rather than any of the local actions $\mathbf{a}_{m,t}$, the joint action $(\mathbf{a}_{1,t},\ldots,\mathbf{a}_{M,t})\in {\cal A}:=\prod_{m\in{\cal M}}\! {\cal A}_m$ determines the transition probability to the next state $\mathbf{s}_{t+1}$ as well as the loss of each agent $\ell_m(\mathbf{s}_t,(\mathbf{a}_{1,t},\ldots,\mathbf{a}_{M,t}))$. 
Accordingly, the per-learner loss in \eqref{opt0} has the following form 
\begin{align}\label{opt0-1}
\small
{\cal L}_m(\bbpi):=\mathbb{E}_{{\cal T}\sim\mathbb{P}(\cdot|\bbpi)}\left[\sum_{t=0}^{\infty}\gamma^t\ell_m\big(\mathbf{s}_t,(\mathbf{a}_{1,t},\cdots,\mathbf{a}_{M,t})\big)\right]
\end{align}
where $\bbpi:=(\bbpi_1, \ldots, \bbpi_M)$ is a joint policy that concatenates $\{\bbpi_m\}_{m\in{\cal M}}$, and the expectation in ${\cal L}_m(\bbpi)$ is taken over all possible joint state-action trajectories, given by ${\cal T}:=\{\mathbf{s}_0,(\mathbf{a}_{1,0},\ldots,\mathbf{a}_{M,0}), \mathbf{s}_1,(\mathbf{a}_{1,1},\ldots,\mathbf{a}_{M,1}),\ldots\}$. 
Replacing the action $\mathbf{a}_t$ in \eqref{opt0} by  the joint action $(\mathbf{a}_{1,t},\cdots,\mathbf{a}_{M,t})$, the collaborative RL problem can be viewed as an instance of distributed RL. 
Different from a single-agent MDP, the local action spaces of different agents can be different, and agents interact with a common environment influenced by all agents.

\textbf{Parallel reinforcement learning.} Different from the multi-agent RL, the parallel RL is motivated by solving a large-scale single-agent RL task that needs to be run in parallel on multiple computing units (a.k.a. workers) \cite{nair2015icml}. 
The advantage of parallel RL is training time reduction and stabilization of the training processes \cite{mnih2016icml}.
Under such a setting, multiple workers typically aim to learn a common policy $\bbpi:{\cal S}\rightarrow \Delta({\cal A})$ for different instances of an  \emph{identical} MDP. 
By different instances of an  {identical} MDP, we mean that each worker $m$ aims to solve an independent MDP characterized by $\left({\cal S}_m, {\cal A}_m, {\cal P}_m, \gamma, \rho_m, \ell_m\right)$. 
In particular, the local action and state spaces as well as the transition probabilities of the workers are the same; that is, ${\cal A}_m= {\cal A}$, ${\cal P}_m = {\cal P}$, and ${\cal S}_m = {\cal S},\, \forall m\in {\cal M}$. However, the losses and the initial state distributions are different across workers, where the initial state distribution of worker $m$ is $\rho_m$, and the loss of worker $m$ is $\ell_m:{\cal S}\times {\cal A}\rightarrow \mathbb{R}$. Nevertheless, they are quantities drawn from the same distributions, which satisfy  $\mathbb{E}[\rho_m(\mathbf{s})]=\rho(\mathbf{s})$ and $\mathbb{E}[\ell_m(\mathbf{s},\mathbf{a})]=\ell(\mathbf{s},\mathbf{a})$ for any $(\mathbf{s},\mathbf{a})\in{\cal S}\times{\cal A}$. 
Thus, the per-learner loss in \eqref{opt0} under the parallel RL can be written as
\begin{align}\label{opt0-2}
\small
{\cal L}_m(\bbpi):=\mathbb{E}_{{\cal T}_m\sim\mathbb{P}(\cdot|\bbpi)}\left[\sum_{t=0}^{\infty}\gamma^t\ell_m(\mathbf{s}_{m,t},\mathbf{a}_{m,t})\right]
\end{align}
where $\mathbf{s}_{m,t}\in {\cal S}_m$, $\mathbf{a}_{m,t}\in{\cal A}_m$ are, respectively, the state and  action of worker $m$, and $\bbpi$ is the  common policy to be learned. The expectation in ${\cal L}_m(\bbpi)$ is taken over all possible state-action trajectories of worker $m$, given by ${\cal T}_m:=\{\mathbf{s}_{m,0},\mathbf{a}_{m,0}, \mathbf{s}_{m,1},\mathbf{a}_{m,1}, \mathbf{s}_{m,2},\mathbf{a}_{m,2},\ldots\}$.
In contrast to the formulation of multi-agent RL in \eqref{opt0-1}, the workers in parallel RL are not coupled by the joint state transition distributions or the loss functions, but they are rather intertwined by employing a common local policy.

As a closing note of this subsection, it is important to emphasize the distinction between the distributed RL setting in this paper and the classic distributed control settings in e.g., \cite{bernstein2002complexity,nayyar2013decentralized}. In distributed control, the global state is not observed by all the learners and thus the problem is more challenging. In our considered distributed RL setting, the state is globally observable but the reward function is private to each agent, and therefore the policy gradient that we will introduce next needs to be acquired in a distributed fashion.

\subsection{Policy gradient methods}
Policy gradient methods have been widely used to solve RL problems with massive and possibly continuous state and action spaces, where the intended solver typically involves function approximation. To overcome the inherent difficulty of learning a function, policy gradient methods restrict the search for the best performing policy over a class of parameterized policies. 
In particular, the policy $\bbpi$ is usually parameterized by $\bbtheta\in\mathbb{R}^d$, which is denoted as $\bbpi(\cdot|\mathbf{s};\bbtheta)$, or $\bbpi(\bbtheta)$ for simplicity.  The commonly used Gaussian policy, for instance, can be parameterized as $\bbpi(\,\cdot\,|\mathbf{s};\bbtheta)=\mathcal{N}(\bbmu(\mathbf{s};\bbtheta),\mathbf{\Sigma}(\mathbf{s}))$, 
where  $\bbmu(\mathbf{s};\bbtheta)$ is a general nonlinear mapping from ${\cal S}$ to ${\cal A}$ parameterized by $\bbtheta$. The mapping $\bbmu(\mathbf{s};\bbtheta)$  can either be a deep neural network with the weight parameters $\bbtheta$, or, a linear function of $\bbtheta$ of the form $\bbmu(\mathbf{s};\bbtheta)=\mathbf{\Phi}(\mathbf{s})\bbtheta$, where  $\mathbf{\Phi}(\mathbf{s})$ is the feature matrix corresponding to the  state $\mathbf{s}$. 
Accordingly, the long-term discounted reward of a parametric policy per agent $m$ is denoted by ${\cal L}_m(\bbtheta):={\cal L}_m(\bbpi(\bbtheta))$. Hence, the distributed RL problem \eqref{opt0} can be rewritten as the following parametric optimization 
\begin{align}\label{opt1}
\small
	\min_{\bbtheta}\!\sum_{m\in{\cal M}}\!{\cal L}_m(\bbtheta)~~{\rm with}~~{\cal L}_m(\bbtheta)\!:=\!\mathbb{E}_{{\cal T}\sim\mathbb{P}(\,\cdot\,|\bbtheta)}\!\!\left[\sum_{t=1}^{\infty}\gamma^t\ell_m(\mathbf{s}_t,\mathbf{a}_t)\right]
\end{align}
where the probability distribution  of a trajectory ${\cal T}$ under  the policy $\bbpi(\bbtheta)$ is denoted as $\mathbb{P}(\cdot|\bbtheta)$. 
The search for an optimal policy can thus be performed by applying the gradient descent-type iterative methods to the parameterized optimization problem \eqref{opt1}. 
By virtue of the \emph{log-trick}, the gradient of each learner's loss ${\cal L}_m(\bbtheta)$ in \eqref{opt1} can be written as \cite{baxter2001jair}
\begin{equation}
	\small
\!\!\nabla {\cal L}_m(\bbtheta)\!=\!\mathbb{E}_{{\cal T}\sim\mathbb{P}(\,\cdot\,|\bbtheta)}\!\!\left[\sum_{t=0}^{\infty}\left(\sum_{\tau=0}^t\nabla\log\bbpi(\mathbf{a}_{\tau}|\mathbf{s}_{\tau};\bbtheta)\!\!\right)\!\!\gamma^t\ell_m(\mathbf{s}_t,\mathbf{a}_t)\right]\!.\label{eq.PG}
\end{equation}
When the MDP model \eqref{eq.tuple} is unknown, or, the expectation in \eqref{eq.PG} is computationally difficult to obtain, the stochastic estimate of the policy gradient \eqref{eq.PG} is often used, that is
\begin{equation}\label{app-eq.PG2}
\small
\hat{\nabla}{\cal L}_m(\bbtheta)=\sum_{t=0}^{\infty}\left(\sum_{\tau=0}^t\nabla\log\bbpi(\mathbf{a}_{\tau}|\mathbf{s}_{\tau};\bbtheta)\right)\gamma^t\ell_m(\mathbf{s}_t,\mathbf{a}_t)
\end{equation}
which is abbreviated as G(PO)MDP policy gradient \cite{baxter2001jair}. The G(PO)MDP policy gradient is an unbiased estimator of the policy gradient, which incurs lower variance than other estimators, e.g., REINFORCE \cite{williams1992}. In our ensuing algorithm design and analysis, we will leverage the G(PO)MDP gradient. 
Nonetheless, the variance of G(PO)MDP gradient can still be high in general, and thus requires small stepsizes and sufficiently many iterations to guarantee convergence. 

\begin{figure}[t]
\vspace{-0.3cm}
\centering
\includegraphics[width=8cm]{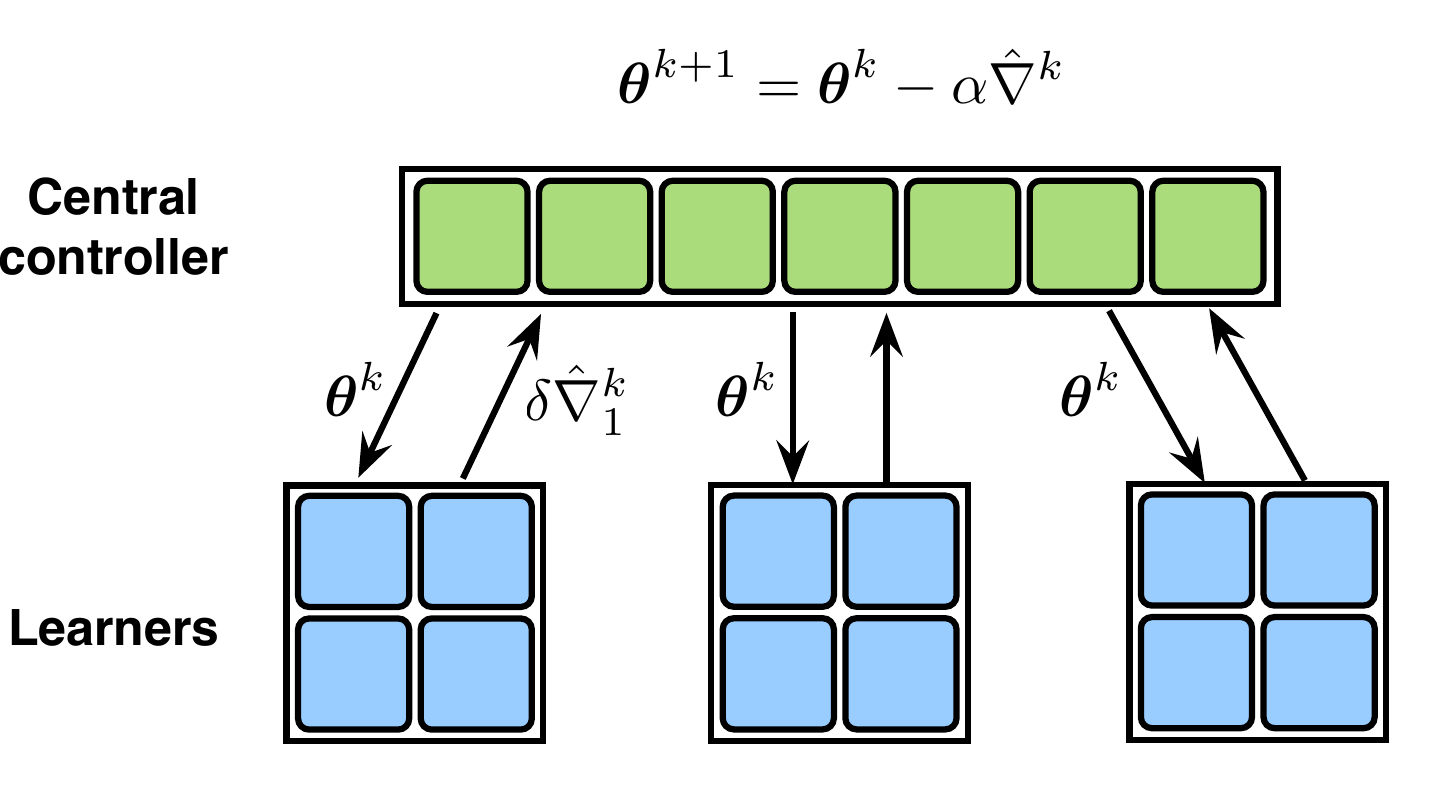}
\vspace*{-8pt}
  \caption{LAPG for communication-efficient distributed RL.}
\label{fig:pswk-diag}
\vspace*{-8pt}
\end{figure}

For the plain-vanilla PG method, a number of needed iterations result in high communication overhead, since all learners' gradients need to be uploaded to form the gradient for the objective in \eqref{opt1}. This motivates the development of communication-efficient schemes to be introduced next. 


\section{Communication-Efficient Distributed Policy Gradient Methods}
Before introducing our approach, we revisit the popular G(PO)MDP-based PG method for solving \eqref{opt1} in the distributed RL setting. 
Define the finite-horizon approximation of the policy gradient \eqref{eq.PG} as
\begin{align}\label{app-eq.POPG}
\small
\!{\nabla}_T{\cal L}_m(\bbtheta)\!=\!\mathbb{E}_{{\cal T}\sim\mathbb{P}({\cal T}|\bbtheta)}\!\left[\sum_{t=0}^T\left(\sum_{\tau=0}^t\nabla\log\bbpi(\mathbf{a}_{\tau}|\mathbf{s}_{\tau};\bbtheta)\right)\!\gamma^t\ell_m(\mathbf{s}_t,\mathbf{a}_t)\!\right].
\end{align}

At iteration $k$, the central controller broadcasts the current policy parameter $\bbtheta^k$ to \emph{all} learners; every learner $m\in{\cal M}$ computes an approximate policy gradient of \eqref{app-eq.POPG} via
\begin{align}\label{eq.batchPG}
	&\small\hat{\nabla}_{\!N,T}{\cal L}_m\big(\bbtheta^k\big):=\\
	&\small\frac{1}{N}\sum_{n=1}^N\sum_{t=0}^T\!\left(\sum_{\tau=0}^t\nabla\log\bbpi(\mathbf{a}_{\tau}^{n,m}|\mathbf{s}_{\tau}^{n,m};\bbtheta^k)\!\right)\!\gamma^t\ell_m(\mathbf{s}_t^{n,m},\mathbf{a}_t^{n,m})\nonumber
\end{align}
where ${\cal T}_T^{n,m}:=(\mathbf{s}_0^{n,m},\mathbf{a}_0^{n,m}, \mathbf{s}_1^{n,m},\mathbf{a}_1^{n,m},\ldots,\mathbf{s}_T^{n,m},\mathbf{a}_T^{n,m})$ is the $n$-th T-slot trajectory (a.k.a. episode) generated at learner $m$; every learner $m$ then uploads $\hat{\nabla}_{\!N,T}{\cal L}_m\big(\bbtheta^k\big)$ to the central controller; and once receiving gradients from all learners, the controller updates the policy parameters via
\begin{equation}\label{eq.gd1}
\small
\bbtheta^{k+1}=\bbtheta^k-\alpha\hat{\nabla}_{\rm PG}^k~~~{\rm with}~~~\hat{\nabla}_{\rm PG}^k:=\! \sum_{m\in{\cal M}}\hat{\nabla}_{\!N,T}{\cal L}_m(\bbtheta^k)
\end{equation}
where $\alpha$ is a stepsize, and $\hat{\nabla}_{\rm PG}^k$ is an aggregated policy gradient with each component received from each learner.   
The policy gradient in \eqref{eq.batchPG} is a mini-batch G(PO)MDP gradient computed by learner $m$ using $N$ batch trajectories $\{{\cal T}_T^{n,m}\}_{n=1}^N$ over $T$ time slots. 
To implement the mini-batch PG update \eqref{eq.gd1}, however, the controller has to communicate with \emph{all} learners to obtain fresh $\{\hat{\nabla}_{\!N,T}{\cal L}_m\big(\bbtheta^k\big)\}$. The mini-batch PG approach for solving \eqref{opt0} is summarized in Algorithm \ref{algo:f-iag}.

  \begin{algorithm}[t]
    \caption{PG for distributed RL}\label{algo:f-iag}
	\begin{algorithmic}[1]
		\State\textbf{Input:}~Stepsize $\alpha>0$, $N$, and $T$.
		\State\textbf{Initialize:}~$\bbtheta^1$.
		\For {$k= 1, 2,\ldots, K$}
		\State Controller \textbf{broadcasts} $\bbtheta^k$ to all learners.
	    \For {learner $m=1, \ldots, M$}
		\State Learner $m$ \textbf{computes} $\hat{\nabla}_{N,T}{\cal L}_m(\bbtheta^k)$.
		\State Learner $m$ \textbf{uploads} $\hat{\nabla}_{N,T}{\cal L}_m(\bbtheta^k)$.
		\EndFor
		\State Controller \textbf{updates} via \eqref{eq.gd1}.
		\EndFor
	\end{algorithmic}
  \end{algorithm}

%

In this context, the present paper puts forth a new policy gradient-based method for distributed RL (as simple as PG) that can \emph{skip} communication at certain rounds, which explains the name \textbf{L}azily \textbf{A}ggregated \textbf{P}olicy \textbf{G}radient (\textbf{LAPG}). With derivations postponed until later, we introduce the LAPG iteration for the distributed RL problem \eqref{opt1} that resembles the PG update \eqref{eq.gd1}, given by
\begin{equation}\label{eq.LAG1}
\small
\bbtheta^{k+1}=\bbtheta^k-\alpha\hat{\nabla}^k~~~~{\rm with}~~~~\hat{\nabla}^k:=\! \sum_{m\in{\cal M}}\hat{\nabla}_{\!N,T}{\cal L}_m\big(\hat{\bbtheta}_m^k\big)
\end{equation}
where each policy gradient $\hat{\nabla}_{\!N,T}{\cal L}_m(\hat{\bbtheta}_m^k)$ is either $\hat{\nabla}_{\!N,T}{\cal L}_m(\bbtheta^k)$, when $\hat{\bbtheta}_m^k=\bbtheta^k$, or an outdated policy gradient that has been computed using an old copy $\hat{\bbtheta}_m^k\neq \bbtheta^k$.
Instead of requesting fresh batch policy gradients from every learner in \eqref{eq.gd1}, our idea here is to obtain $\hat{\nabla}^k$ by refining the previous aggregated policy gradient $\hat{\nabla}^{k-1}$; e.g., using only the new gradients from the learners in ${\cal M}^k$, while reusing the outdated gradients from the remaining learners.
Therefore, with $\small\hat{\bbtheta}_m^k\!:=\!\bbtheta^k,\,\forall m\!\in\!{\cal M}^k,~\hat{\bbtheta}_m^k\!:=\!\hat{\bbtheta}_m^{k-1}\!\!,\,\forall m\!\notin\!{\cal M}^k$, after rearranging terms, \textbf{LAPG recursion} \eqref{eq.LAG1} becomes
\begin{equation}\label{eq.LAG2}
\small
	\bbtheta^{k+1}=\bbtheta^k-\alpha \hat{\nabla}^{k-1} - \alpha\!\!\sum_{m\in{\cal M}^k}\delta\hat{\nabla}^k_m
\end{equation}
where $\small\delta\hat{\nabla}^k_m:=\hat{\nabla}_{N,T}{\cal L}_m(\bbtheta^k)-\hat{\nabla}_{N,T}{\cal L}_m(\hat{\bbtheta}_m^{k-1})$ denotes the \emph{innovation} between two evaluations of $\hat{\nabla}_{N,T}{\cal L}_m$ at the current policy parameter $\bbtheta^k$ and the old copy $\hat{\bbtheta}_m^{k-1}$. 
In this way, if the central controller stores the previous $\hat{\nabla}^{k-1}$, learners in ${\cal M}^k$ only need to upload the innovation between two policy gradient evaluations; see Figure \ref{fig:pswk-diag}. 
Here the old copies for evaluating policy gradient at each learner can be different, depending on the most recent iteration that each learner uploads its fresh policy gradient.

To this point, a myopic approach to minimizing per-iteration communication is to include as few learners in ${\cal M}^k$ as possible. 
However, it will turn out that such a simple selection will lead to many more iterations that may in turn increase the total number of needed communication rounds. 
A more principled way is to guide the communication selection according to learners' optimization progress, which leads to the LAPG's selection rule presented at the end of this section. 
The first step of deriving such principle is to characterize the optimization progress of LAPG as follows. 
\begin{lemma}[LAPG descent lemma]\label{lemma1}
Assume ${\cal L}(\bbtheta):=\sum_{m\in{\cal M}}{\cal L}_m(\bbtheta)$ is $L$-smooth, and $\bbtheta^{k+1}$ is generated by running one-step LAPG iteration \eqref{eq.LAG1} given $\bbtheta^k$. If the stepsize is selected as $\alpha\leq 1/L$, then the objective values satisfy 
\begin{align}\label{eq.lemma1}
 & {\cal L}(\bbtheta^{k+1})\!-\! {\cal L}(\bbtheta^k)\!\leq\nonumber\\ 
&-\frac{\alpha}{2}\left\|\nabla {\cal L}(\bbtheta^k)\right\|^2\!\!+\frac{3\alpha}{2}\Big\|\nabla_T {\cal L}(\bbtheta^k)-\nabla{\cal L}(\bbtheta^k)\Big\|^2\nonumber\\
 & +\frac{3\alpha}{2}\Bigg\|\!\sum_{m\in{\cal M}^k_c}\!\!\delta\hat{\nabla}^k_m\Bigg\|^2\!\!+\!\frac{3\alpha}{2}\Big\|\hat{\nabla}_{N,T}{\cal L}\big(\bbtheta^k\big)\!-\!\nabla_T {\cal L}(\bbtheta^k)\Big\|^2\nonumber\\
& +\left(\frac{L}{2}-\frac{1}{2\alpha}\right)\!\left\|\bbtheta^{k+1}\!-\!\bbtheta^k\right\|^2
	\end{align}
	where $\delta\hat{\nabla}^k_m$ is defined in \eqref{eq.LAG2}, $\hat{\nabla}_{N,T}{\cal L}\big(\bbtheta^k\big):=\sum_{m\in{\cal M}}\hat{\nabla}_{N,T}{\cal L}_m\big(\bbtheta^k\big)$, and ${\cal M}^k_c:={\cal M}\backslash{\cal M}^k$ is the set of learners that \emph{do not} upload at iteration $k$.
\end{lemma}
\begin{IEEEproof}
See Appendix \ref{pf.Lemma1}.
\end{IEEEproof}

In Lemma \ref{lemma1}, the first term on the right hand side of \eqref{eq.lemma1} drives the descent in the objective of distributed RL, while the finite-horizon gradient approximation error (the second term), the error induced by skipping communication (the third term), as well as the variance of stochastic policy gradient (the fourth term) increase the distributed RL objective thus impede the optimization progress. 
Intuitively, the error induced by skipping communication should be properly controlled so that it is small or even negligible relative to the magnitude of policy gradients that drives the optimization progress, and also the variance-induced error of policy gradients that appears in the PG-type algorithms \cite{papini18icml}. 

To account for these error terms in the design of our algorithm, we first approximate $\small\|\nabla {\cal L}(\bbtheta^k)\|^2$ in \eqref{eq.lemma1} by the differences of successive policy parameters $\small\sum_{d=1}^D \frac{ \xi}{\alpha^2}\big\|\bbtheta^{k+1-d}-\bbtheta^{k-d}\big\|^2$, where $\xi$ is the pre-selected constant and $D$ is a pre-selected interval length, and then quantify the variance of using mini-batch policy gradient estimation in the next lemma. Note that to obtain the finite-sample analysis of LAPG, we borrow powerful tools from the celebrated probably approximately correct (PAC) learning framework, e.g., \cite{vapnik2013nature,shalev2010learnability,schapire1990strength}.

  \begin{algorithm}[t]
    \caption{LAPG for distributed RL}\label{algo:f-iag2}
	\begin{algorithmic}[1]
		\State\textbf{Input:}~Stepsize $\alpha>0$, $\{\xi_d\}$, $N$ and $T$.
		\State\textbf{Initialize:}~$\bbtheta^1, \hat{\nabla}^0, \{\hat{\bbtheta}^0_m,\forall m\}$.
		\For {$k= 1, 2,\ldots, K$}
	    \State Controller \textbf{broadcasts} the policy parameters.	
	    \For {learner $m=1, \ldots, M$}
	    		\State Learner $m$ \textbf{computes} $\hat{\nabla}_{N,T}{\cal L}_m(\bbtheta^k)$.
	    \If {learner $m$ satisfies the condition \eqref{eq.trig-cond}}
		\State Learner $m$ \textbf{uploads} $\delta\hat{\nabla}^k_m$ and saves $\hat{\bbtheta}_m^k=\bbtheta^k$.
		\Else \State{No actions at learner $m$.}
	    \EndIf
		\EndFor
		\State Controller \textbf{updates} the global policy via \eqref{eq.LAG2}.
		\EndFor
	\end{algorithmic}
  \end{algorithm}
  

\begin{lemma}[PG concentration]\label{lemma.PGwhp}
Under Assumptions 1 and 2, there exists a constant $V_m$ depending on $G, \gamma, \bar{\ell}_m$ such that given $K$ and $\delta\in(0,1)$, with probability at least $1-\delta/K$, for any $\bbtheta$, we have that
\begin{equation}\label{eq.agtvar}
\Big\|\hat{\nabla}_{N,T}{\cal L}_m\big(\bbtheta\big)-\nabla_T {\cal L}_m(\bbtheta)\Big\|^2\leq \frac{2\log (2K/\delta)V_m^2}{N}:=\sigma^2_{m,N,\delta/K}	
\end{equation}
where $\small\hat{\nabla}_{N,T}{\cal L}_m\big(\bbtheta\big)$ and $\small\nabla_T {\cal L}_m(\bbtheta)$ are the stochastic policy gradient \eqref{eq.batchPG}, and the full policy gradient for the $T$-slot truncated objective \eqref{opt1}, namely, $\small\mathbb{E}_{{\cal T}\sim\mathbb{P}(\,\cdot\,|\bbtheta)}\big[\sum_{t=1}^{T}\gamma^t\ell_m(\mathbf{s}_t,\mathbf{a}_t)\big]$. 
\end{lemma}
\begin{IEEEproof}
 See Appendix \ref{pf.lemma2}.
\end{IEEEproof}

Lemma \ref{lemma.PGwhp} suggests that the deviation of the stochastic policy gradient from the true one in \eqref{eq.lemma1} can be bounded with high probability by $\sigma^2_{m,N,\delta/K}$, which mainly depends on the number of trajectories $N$, and has logarithmic dependence on the confidence $\delta$, as well as the number of iterations $K$.

Building upon Lemmas \ref{lemma1} and \ref{lemma.PGwhp}, we will include the learner $m$ in ${\cal M}^k$ of \eqref{eq.LAG2} only if its current policy gradient has enough innovation relative to the most recently uploaded one; that is, it satisfies the following \textbf{LAPG condition}:
\begin{equation}\label{eq.trig-cond}
\!\!\left\|\delta\hat{\nabla}^k_m\right\|^2\geq \frac{\xi}{\alpha^2M^2}\sum_{d=1}^D\!\left\|\bbtheta^{k+1-d}-\bbtheta^{k-d}\right\|^2\!+6\sigma^2_{m,N,\delta/K}
\end{equation}
where $\xi$ are constant weights, and $\sigma^2_{m,N,\delta/K}$ is the variance of the policy gradient in \eqref{eq.agtvar}.
The values of $\{ \xi\}$ and $D$ are hyper-parameters and can be optimized case-by-case, while the variance $\sigma^2_{m,N,\delta/K}$ can be estimated on-the-fly in simulations.
In a nutshell, LAPG for solving the distributed RL problem \eqref{opt1} is summarized in Algorithm \ref{algo:f-iag2}.

\vspace{0.1cm}
Regarding our LAPG method, two remarks are in order.

\begin{remark}[LAPG implementation]
	By recursively updating the gradients in \eqref{eq.LAG1} and the lagged condition in \eqref{eq.trig-cond}, implementing LAPG is as simple as PG (the pseudocode in the supplementary material). The only additional complexity comes from storing the most recently uploaded policy gradient {\small $\hat{\nabla}_{N,T}{\cal L}_m(\hat{\bbtheta}_m^k)$} and checking the LAPG communication condition \eqref{eq.trig-cond}. 
Despite its simplicity, we will demonstrate that using lagged policy gradients in distributed RL can cut down a portion of unnecessary  communication among learners.
\end{remark}

\begin{remark}[Beyond LAPG]
Compared with existing efforts for improving PG in single-agent settings such as the trust region PG \cite{schulman2015icml}, the deterministic PG \cite{silver2014dpg}, and the variance-reduced PG \cite{papini18icml}, LAPG is not orthogonal to any of them.
Instead, LAPG points out an alternate direction for improving communication efficiency of solving distributed RL, and can be combined with these methods to develop even more powerful distributed RL algorithms. 
Extension to the natural policy gradient version of LAPG is also possible to remove the dimensional dependence and accelerate the convergence. 
While the current LAPG algorithm requires continuously monitoring the communication condition \eqref{eq.trig-cond}, an important extension is to allow intermittent monitoring of the condition. This can be potentially achieved by predicting the maximum change of parameters during a fixed interval, and then setting a timer to wake up the learner after a fixed number of iterations. Finally, the current analysis of LAPG in Lemma \ref{lemma.PGwhp} requires the unbiased estimate of the gradient, which is more suitable for the episodic setting. 
It is also valuable to extend the analysis to the case where LAPG is implementing in the continuing tasks.
However, these extensions go beyond the scope of this paper, and will be pursued in future work. 
\end{remark}
 

\section{Main Results}
In this section, we present the main theorems quantifying the performance of LAPG.
Before that, we introduce several assumptions that serve as stepping stones for the analysis. The complete proofs of all the lemmas and theorems can be found in the supplementary document. 

\noindent\textbf{Assumption 1}:
\emph{For each state-action pair $(\mathbf{s},\mathbf{a})$, the loss $\ell_m(\mathbf{s},\mathbf{a})$ is bounded as $\ell_m(\mathbf{s},\mathbf{a})\in[0,\bar{\ell}_m]$, and thus for any $\bbtheta$, the per-learner loss is bounded as $\small{\cal L}_m(\bbtheta)\in [0,\bar{\ell}_m/(1-\gamma)]$.} 
\vspace{-0.1cm}

\noindent\textbf{Assumption 2}:
\emph{For each state-action pair $(\mathbf{s},\mathbf{a})$, and any parameter $\bbtheta\in \mathbb{R}^d$, there exist constants $G$ and $F$ such that}
\begin{equation}\label{eq.assp2}
\small
	\left\|\nabla\log \bbpi(\mathbf{a}|\mathbf{s};\bbtheta)\right\|\leq G~~~{\rm and}~~~\left|\frac{\partial^2}{\partial \theta_i \partial \theta_j}\log \bbpi(\mathbf{a}|\mathbf{s};\bbtheta)\right|\leq F
\end{equation}
where $\theta_i$ and $\theta_j$ denote the $i$th and $j$th entries of $\bbtheta$.

Assumption 1 requires boundedness of the instantaneous loss and thus the discounted cumulative loss, which is natural and commonly assumed in analyzing RL algorithms, e.g., \cite{papini18icml,zhang18icml,baxter2001jair}.  Assumption 2 requires the score function and its partial derivatives to be bounded, which can be also satisfied by a wide range of stochastic policies \cite{papini18icml,zhang20pg}. 
As we will see next, Assumptions 1 and 2 are sufficient to guarantee the smoothness of the objective function in \eqref{opt1}.

\begin{lemma}[smoothness in cumulative losses]\label{lemma-smooth}
Under Assumptions 1 and 2, the cumulative loss ${\cal L}_m(\bbtheta)$ for learner $m$ is $L_m$-smooth, that is, for any policy parameters $\bbtheta_1, \bbtheta_2\in \mathbb{R}^d$, $\|\nabla {\cal L}_m(\bbtheta_1)-\nabla{\cal L}_m(\bbtheta_2)\|\leq L_m\big\|\bbtheta_1-\bbtheta_2\big\|$ with
\begin{equation}\label{eq.smooth_lm}
\small
L_m:=\left(F+G^2+\frac{2\gamma G^2}{1-\gamma}\right)\frac{\gamma\bar{\ell}_m}{(1-\gamma)^2}
\end{equation}
where $\bar{\ell}_m$ is the upper bound of the loss for learner $m$ in Assumption 1, and $F$, $G$ are constants bounding the score function in \eqref{eq.assp2}. 
Likewise, the aggregated loss ${\cal L}(\bbtheta)$ is $L$-smooth with $L:=\sum_{m\in{\cal M}}L_m$.
\end{lemma}
\begin{IEEEproof}
The proof of smoothness is standard. To be self-contained, we also provide the proof in Appendix \ref{pf.lemma3}.
\end{IEEEproof}

The smoothness of the objective function is critical in the convergence analyses of many nonconvex optimization algorithms. Building upon Lemma \ref{lemma-smooth}, 
LAPG can guarantee the following convergence result.
\begin{theorem}[iteration complexity]\label{theorem1}
Under Assumptions 1 and 2, if the stepsize $\alpha$ and the parameters $\xi$ in the LAPG condition \eqref{eq.trig-cond} are chosen such that $\alpha \leq \frac{1}{L}\big(1-3D \xi\big)$, 
and the constants $T$, $K$, and $N$ are chosen to satisfy
\begin{equation}\label{eq.theorem1-0}
\small
T={\cal O} (\log(1/\epsilon)),~~~ K={\cal O}(1/\epsilon), ~~~{\rm and}~~~N={\cal O} (\log(K/\delta)/\epsilon)
\end{equation}
then with probability at least $1-\delta$, the iterates $\{\bbtheta^k\}$ generated by LAPG satisfy
\begin{equation}\label{eq.theorem1-1}
\small
\frac{1}{K}\!\sum_{k=1}^K\big\|\nabla {\cal L}(\bbtheta^k)\big\|^2\!\leq\! \frac{2({\cal L}(\bbtheta^1)-{\cal L}(\bbtheta^*))}{\alpha K}+3\sigma_T^2+21\sigma^2_{N,\delta/K}\!\leq\!\epsilon
\end{equation}
where $\sigma_T$ and $\sigma_{N,\delta/K}^2:=M\sum_{m\in{\cal M}}\sigma_{m,N,\delta/K}^2$ are some constants depending on $T, N, F, G,\gamma, \{\bar{\ell}_m\}$.
\end{theorem}
\begin{IEEEproof}
See Appendix \ref{pf.theorem1}.
\end{IEEEproof}

Theorem \ref{theorem1} asserts that even with the adaptive communication rules, LAPG can still achieve sublinear convergence to the stationary point of \eqref{opt1} as plain-vanilla PG. 

Regarding the communication complexity, it would be helpful to first estimate each learner's frequency of activating the communication condition \eqref{eq.trig-cond}. 
Ideally, we want those learners with a small reward (thus a small smoothness constant in \eqref{eq.smooth_lm}) to communicate with the controller less frequently. This intuition will be formally captured in the next lemma.
\begin{lemma}[lazy communication]\label{lemma4}
Under Assumptions 1 and 2, define the task hardness of every learner $m$ as $\mathds{H}(m):=L_m^2/L^2$.  If for a given $d$, the hardness of the learner $m$ satisfies
 \begin{equation}\label{eq.lemma4}
 \small
\mathds{H}(m)\leq \frac{ \xi}{3d\alpha^2 L^2 M^2}:=\gamma_d
 \end{equation}
 then it uploads to the controller at most $1/(d+1)$ fraction of time, with probability at least $1-2\delta$.
\end{lemma}
\begin{IEEEproof}
	The proof can be found in Appendix \ref{pf.lemma4}.
\end{IEEEproof}

Lemma \ref{lemma4} implies that the communication frequency of each learner is proportional to its task hardness. 
In addition, choosing larger trigger constants $\{ \xi\}$ and a smaller stepsize $\alpha$ will reduce the communication frequencies of all learners. 
However, such choice of parameters will generally require many more iterations to achieve a desirable accuracy. To formally characterize the overall communication overhead of solving distributed RL, we define the communication complexity of solving the distributed RL problem \eqref{opt0} as the number of needed uploads to, with high probability, achieve $\epsilon$-policy gradient error; e.g., $\min_{k=1,\cdots,K}	 \|\nabla{\cal L}(\bbtheta^k)\|^2\leq \epsilon$. 

Building upon Theorem \ref{theorem1} and Lemma \ref{lemma4}, the communication complexity is established next.

\begin{theorem}[communication complexity]\label{theorem2}
Under Assumptions 1 and 2, define the constant $\Delta\mathbb{C}(h;\{\gamma_d\})$ as  
\begin{equation}\label{eq.prop5}
\small
\Delta\mathbb{C}(h;\{\gamma_d\}):=\sum_{d=1}^D\Big(\frac{1}{d}-\frac{1}{d+1}\Big)h\left(\gamma_d\right)	
\end{equation}
where $h$ is the cumulative density function of the learners' task hardness, given by $h(\gamma)\!:=\!\frac{1}{M}\!\sum_{m\in{\cal M}}\mathds{1}(\mathds{H}(m)\leq \gamma)$.
With the communication complexity of LAPG and PG denoted as $\mathbb{C}_{\rm LAPG}(\epsilon)$ and $\mathbb{C}_{\rm PG}(\epsilon)$, if the parameters are chosen as in \eqref{eq.theorem1-0}, with probability at least $1-4\delta$, we have 
\begin{equation}\label{prononconvex-r1}
\small
\mathbb{C}_{\rm LAPG}(\epsilon)\leq\left(1-\Delta\mathbb{C}(h;\{\gamma_d\})\right)\!\frac{\mathbb{C}_{\rm PG}(\epsilon)}{(1-3D \xi)}.
\end{equation}
Choosing the parameters as in Theorem \ref{theorem1}, and if for the heterogeneity function $h(\gamma)$ there exists $\gamma'$ such that 
\begin{equation}
\small
(D+1)D M^2\gamma'< h(\gamma') 
\end{equation}
then we have that $\mathbb{C}_{\rm LAPG}(\epsilon)<\mathbb{C}_{\rm PG}(\epsilon)$.
\end{theorem}

\begin{IEEEproof}
 See Appendix \ref{pf.theorem2}.
\end{IEEEproof}

\begin{figure}[t]
\centering
\includegraphics[width=8cm]{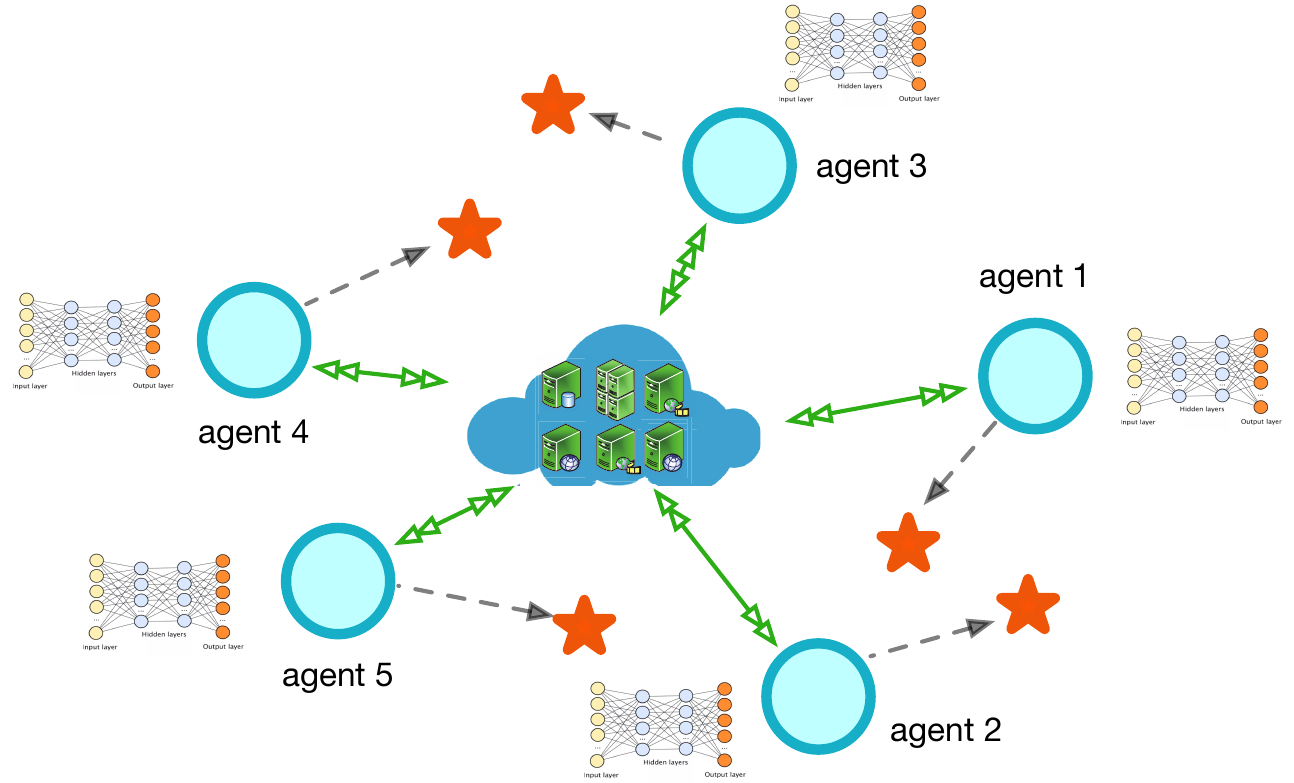}
  \caption{Multi-agent cooperative navigation task. The blue circles represent the agents, the stars represent the landmarks, the green arrows represent the agent-cloud communication links, and the gray arrows direct the target landmark each agent aims to cover.}
\label{fig:rl_task}
\vspace{-0.4cm}
\end{figure}
By carefully designing our communication selection rule, Theorem \ref{theorem2} demonstrates that the overall communication of LAPG is less than that of PG, provided that the reward functions (thus the smoothness constants) of each learner are very heterogeneous. 
 To see this point, consider the case where $L_m={\cal O}(1),\,\forall m=1,\ldots,M-1$, and $L_M=L={\cal O}(M^2)$. Thus, we have $h(\gamma)\geq 1- \frac{1}{M}$, if $\gamma\geq 1/L^2$. If we choose $D= M$ and $\xi=M^2D/(6L^2)={\cal O}(1/D)$ such that $\gamma_D \geq 1/L^2$, then we have (cf. \eqref{prononconvex-r1})
\begin{equation}
\small
\!\!\!\frac{\mathbb{C}_{\rm LAPG}(\epsilon)}{\mathbb{C}_{\rm PG}(\epsilon)} \leq \frac{ 1-(1- \frac{1}{D})(1- \frac{1}{M})}{1-1/2}\approx \frac{M+D}{MD}={\cal O}\left(\frac{1}{M}\right).
\end{equation}
In this case, LAPG roughly requires ${\cal O}(1/M)$ number of communication rounds of plain-vanilla PG.

While the improved communication complexity in Theorem \ref{theorem2} builds on slightly restrictive dependence on the problem parameters, the LAPG's empirical performance goes beyond the worst-case theoretical analysis presented. 
In numerical tests, we find that using a default value $\xi=1/D$ is sufficient to demonstrate the performance gain of LAPG over PG.

\begin{figure}[t]
\centering
\begin{tabular}{cc}
\hspace{-0.25cm}\includegraphics[width=4.8cm]{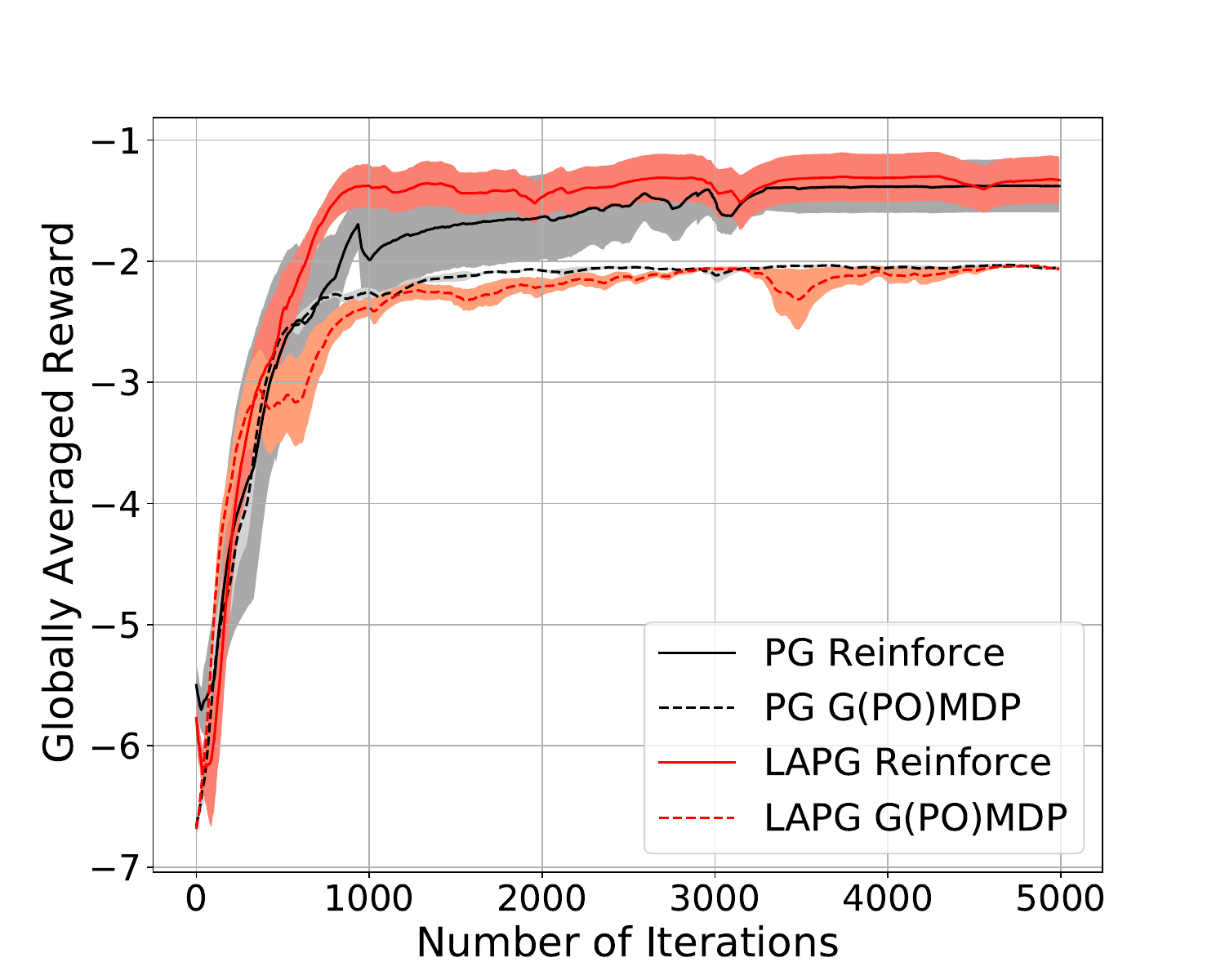}&
\hspace{-0.8cm}
\includegraphics[width=4.8cm]{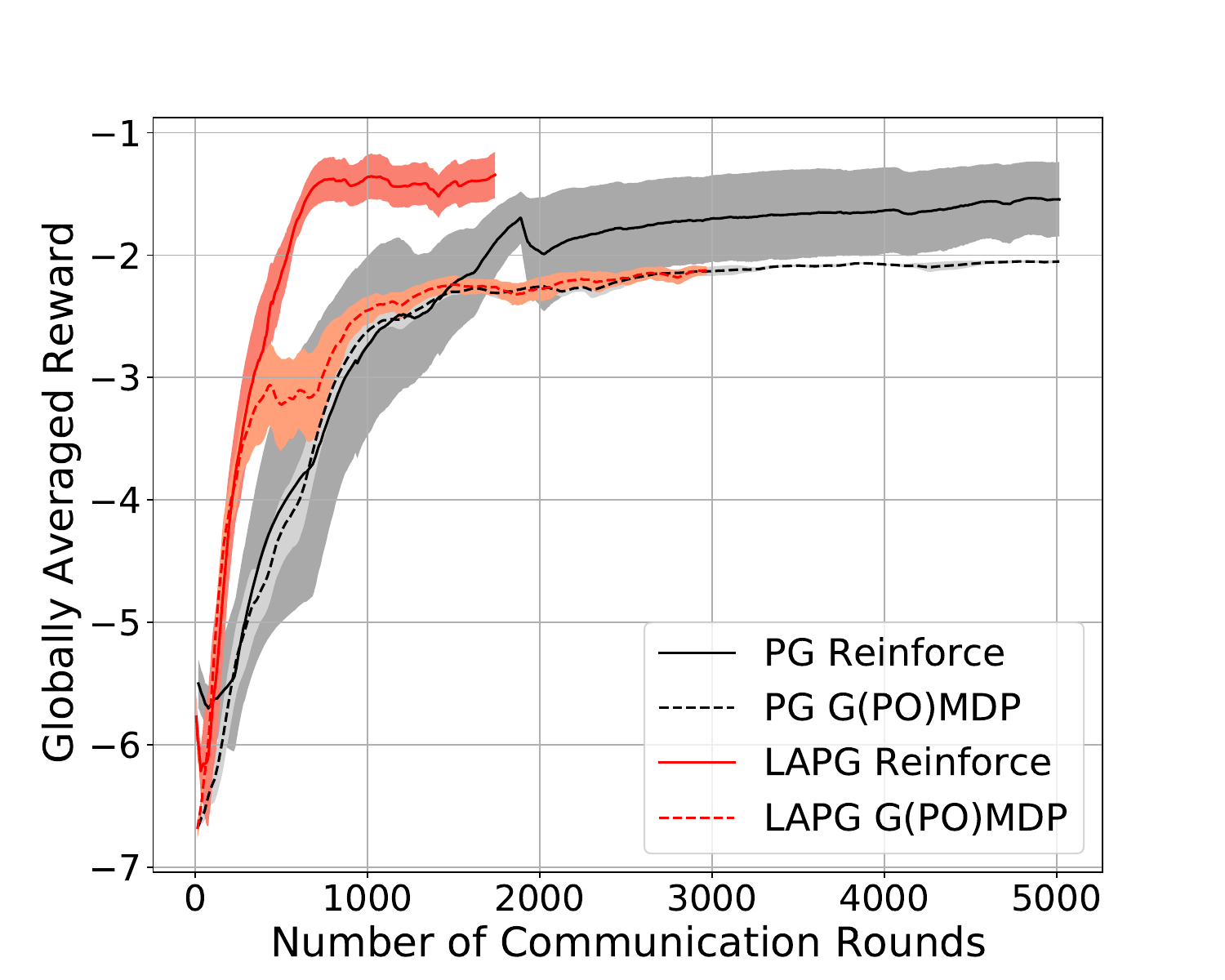}
\end{tabular}
\vspace*{-0.1cm}
  \caption{Iteration and communication complexity of two-agent parallel RL. The shaded region in all the figures represents the reward distribution of each scheme within one standard deviation of the mean.}
\label{fig:rltest1}
\vspace*{-0.4cm}
\end{figure}

\begin{figure}[t]
\centering
\begin{tabular}{cc}
\hspace{-0.25cm}\includegraphics[width=4.8cm]{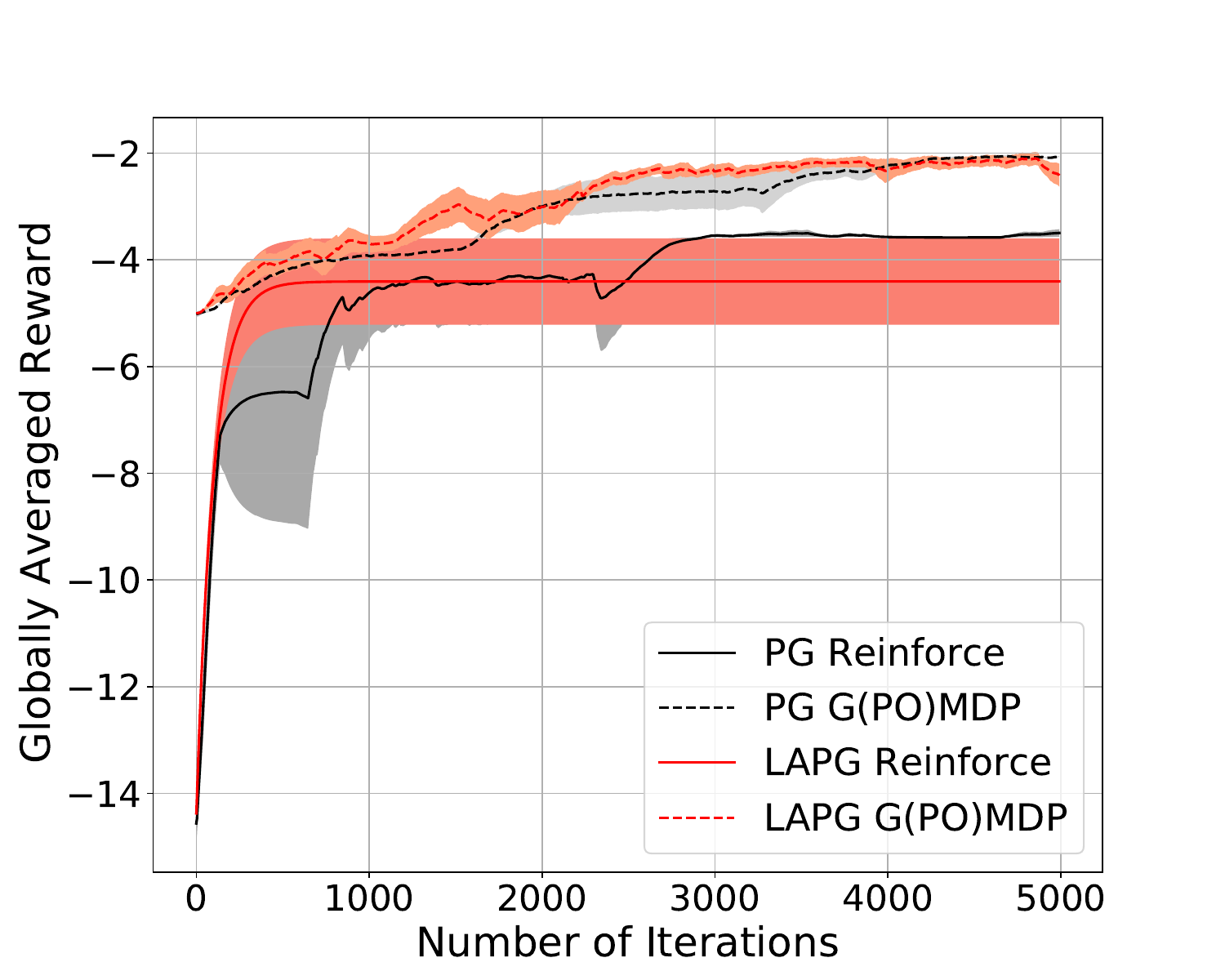}&
\hspace{-0.8cm}
\includegraphics[width=4.8cm]{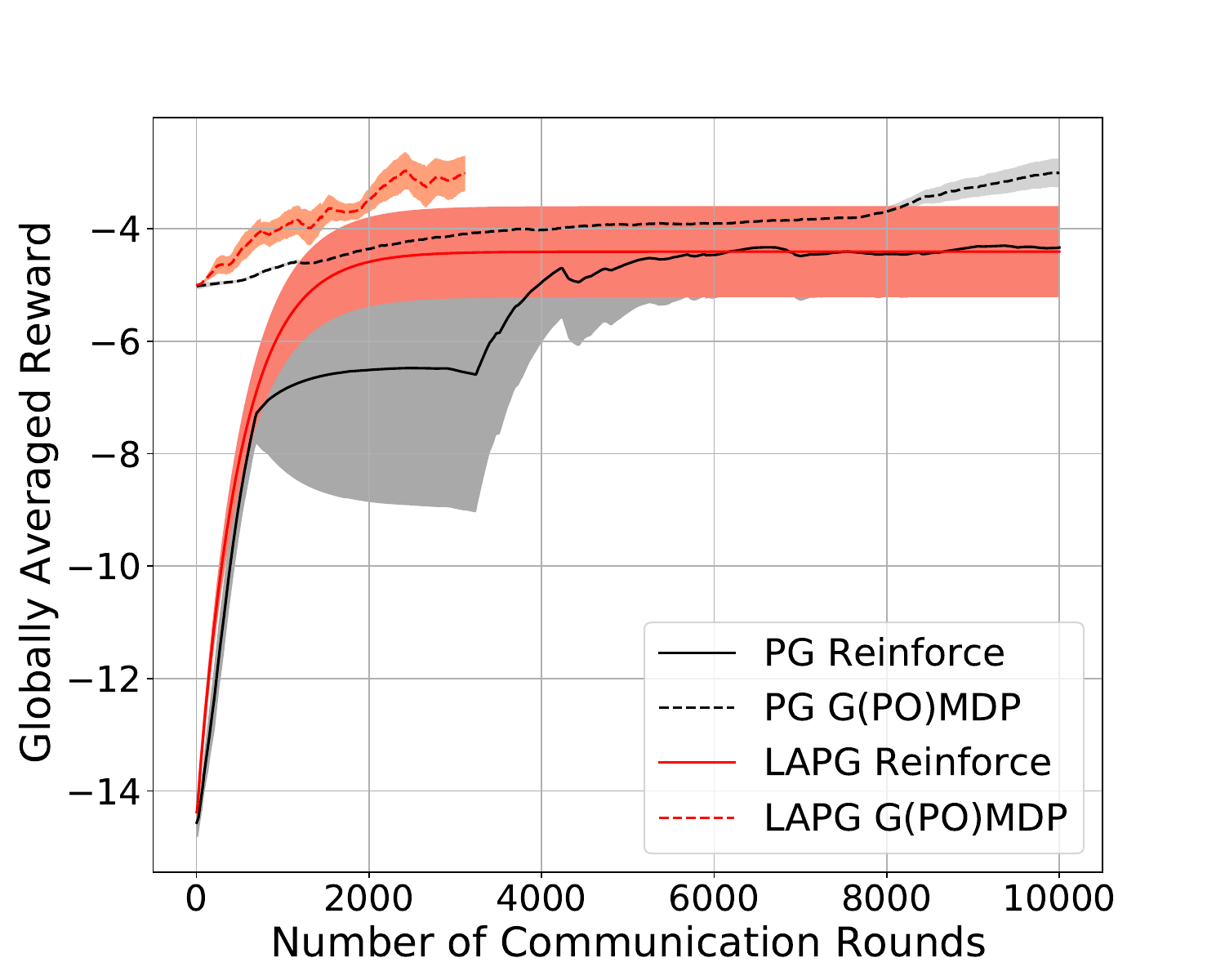}
\end{tabular}
\vspace*{-0.1cm}
  \caption{Iteration and communication complexity in five-agent parallel RL. }
\label{fig:rltest2}
\vspace*{-0.4cm}
\end{figure}

\section{Numerical Tests}
To validate the theoretical results, this section reports the empirical performance of LAPG in both the parallel RL and multi-agent RL task, as two examples of distributed RL. All experiments were performed using Python 3.6 on an Intel i7 CPU @ 3.4 GHz (32 GB RAM) desktop. 
Throughout this section, we consider the simulation environment of the \emph{Cooperative Navigation} task in \cite{lowe2017nips}, which builds on the popular OpenAI Gym paradigm  \cite{openai2016}. 
In this RL environment, $M$ agents aim to reach a set of $M$ landmarks through
physical movement, which is controlled by a set of five actions \{\emph{stay, left, right, up, down}\}. 
Agents are connected to a remote central coordinator, and are rewarded based on the proximity of their position to the one-to-one associated landmark; see the depiction in Figure \ref{fig:rl_task}.

\begin{figure}[t]
\centering
\begin{tabular}{cc}
\hspace{-0.25cm}\includegraphics[width=4.8cm]{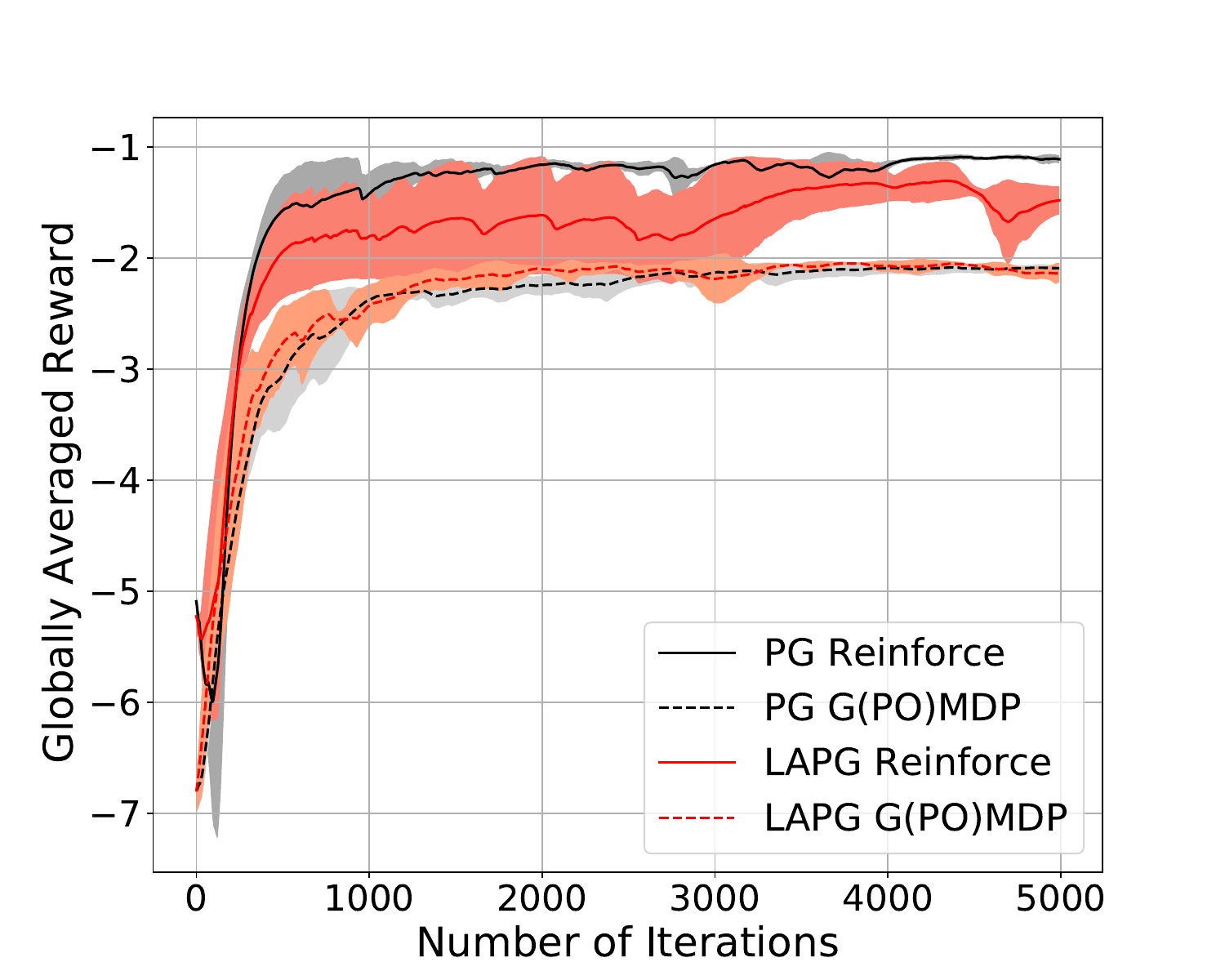}&
\hspace{-0.8cm}
\includegraphics[width=4.8cm]{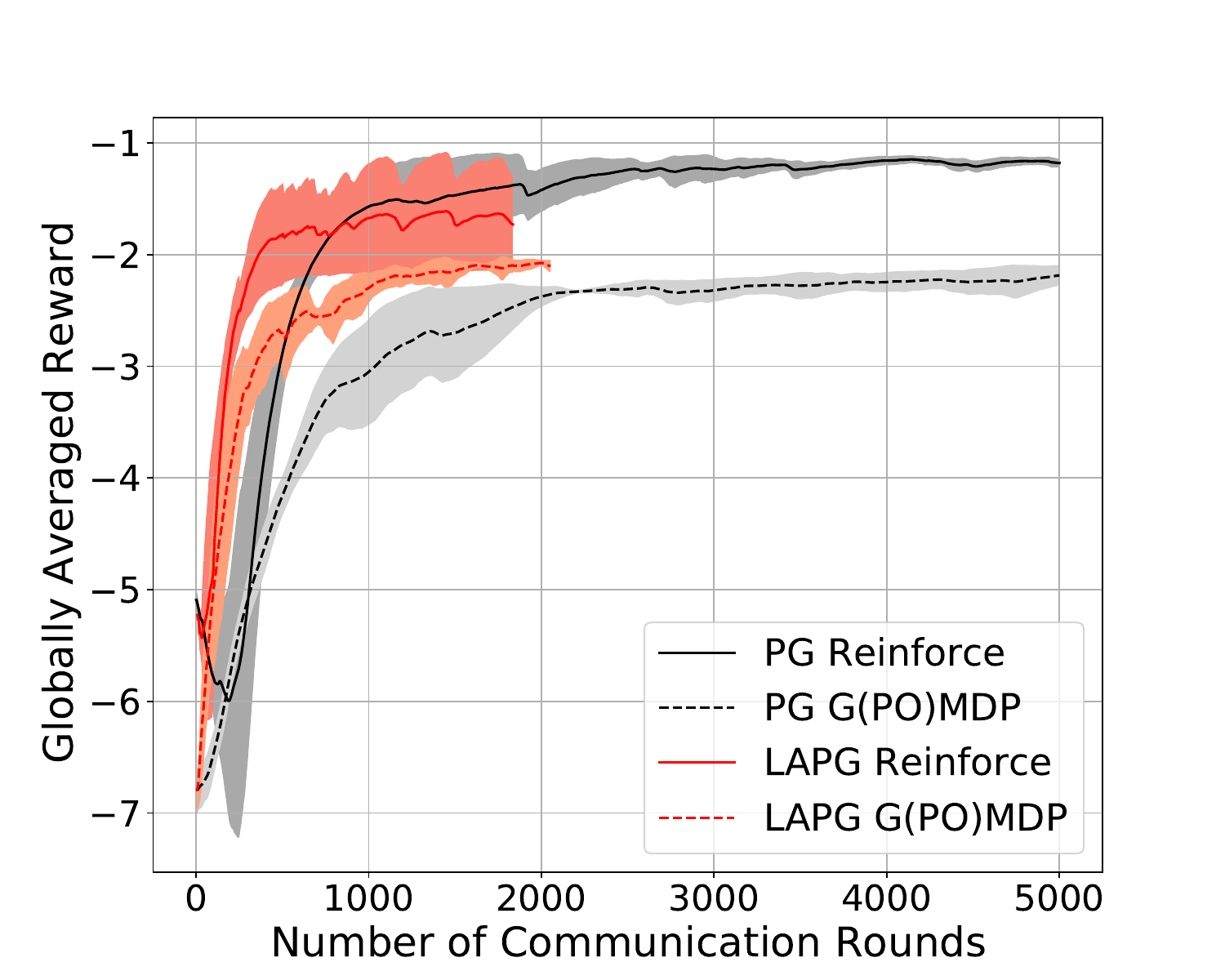}
\end{tabular}
\vspace*{-0.1cm}
  \caption{Iteration and communication complexity of two-agent multi-agent RL. }
\label{fig:rltest3}
\vspace*{-0.4cm}
\end{figure}

\begin{figure}[t]
\centering
\begin{tabular}{cc}
\hspace{-0.25cm}\includegraphics[width=4.9cm]{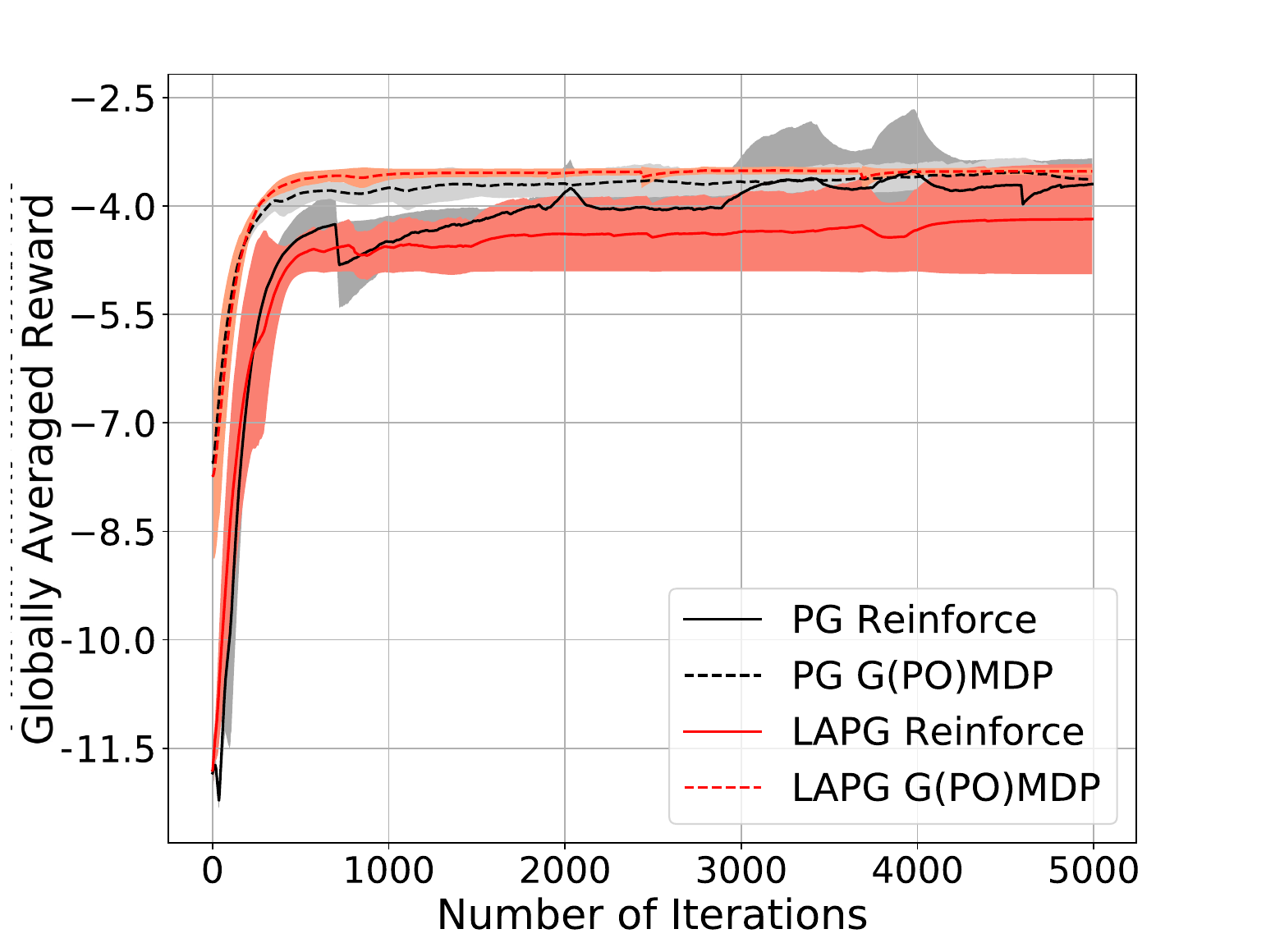}&
\hspace{-0.8cm}
\includegraphics[width=4.75cm]{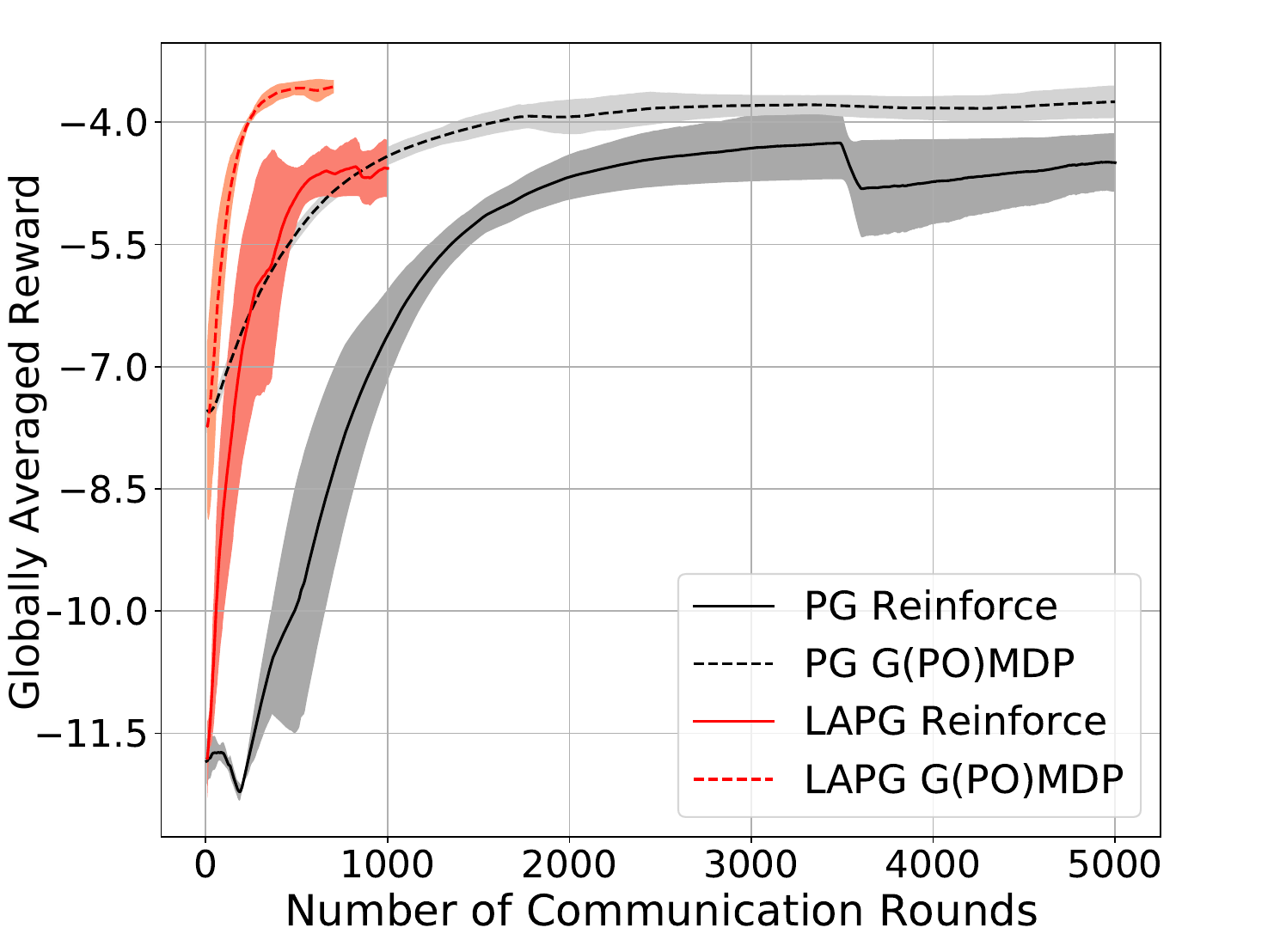}
\end{tabular}
\vspace*{-0.1cm}
  \caption{Iteration and communication complexity of five-agent multi-agent RL.}
\label{fig:rltest4}
\vspace*{-0.4cm}
\end{figure}

In the simulation, we modify the environment in \cite{lowe2017nips} as follows: i) we assume the state is globally observable, i.e., the position and velocity of other agents in a two-dimensional grid are observable to each agent; and, ii) each agent has a certain target landmark to cover, and the individual reward is determined by the proximity to that certain landmark, as well as the penalty of collision with other agents. 
In this way, the reward function varies among agents, and the individual reward of an agent also depends on the other agents' movement, which is consistent with the multi-agent RL formulation \eqref{opt0-1}. 
The reward is further scaled by different positive coefficients, representing
the heterogeneity (e.g., different priority) of different agents. The collaborative goal of the agents is to maximize the network averaged long-term reward so as to reduce distances to the landmark and avoid collisions. 
We consider two settings of the above environment: $M=2$ and $M=5$ agents. 
For $M=2$ agents, the targeted local policy of each agent $\bbpi_m(\bbtheta_m)$ is parameterized by a three-layer neural network, where the first and the second hidden layers contain 30 and 10 neural units with ReLU as the activation function, and the output layer is the softmax operator. For $M=5$, the targeted policy is again parameterized by a three-layer neural network, where the first and the second hidden layers contain 50 and 20 neurons. 

We implement LAPG using both G(PO)MDP and REINFORCE gradient estimators, and compare it with the vanilla  G(PO)MDP- and REINFORCE-based PG method. The discounting factor in the cumulative loss is $\gamma=0.99$ in all the tests. For each episode, both algorithms terminate after $T=20$ iterations.
We run in total $N=10$ batch episodes in each Monte Carlo run, and report the globally averaged reward from 5 Monte Carlo runs. 
To accelerate the training of neural networks used in policy parameterization, both LAPG and PG are implemented using heavy-ball based momentum update, where the stepsize and the momentum factor are set as 0.01 and 0.6, respectively. 

%
%

In Figures \ref{fig:rltest1} and \ref{fig:rltest2}, we compare the performance of LAPG and PG in the parallel RL task with $M=2$ and $M=5$, respectively. For $M=2$, each learner solves an independent \emph{Cooperative Navigation} task with 2 agents; for $M=5$, each learner solves an independent \emph{Cooperative Navigation} task with 5 agents. 
As shown in the figures, both G(PO)MDP- and REINFORCE-based LAPG converge within the same number of iterations as G(PO)MDP- and REINFORCE-based PG. When it comes to the number of communication rounds, LAPG requires significantly smaller amount than PG in both $M=2$ and $M=5$ cases. Comparing different policy gradient estimators, we find that REINFORCE-based estimators generally have higher variance than G(PO)MDP-based estimators.


 In Figures \ref{fig:rltest3} and \ref{fig:rltest4}, we compare the performance of LAPG and PG in the multi-agent RL task. Different from the parallel RL setting, in this multi-agent RL case, all the learners solve the shared \emph{Cooperative Navigation} task. 
In both $M=2$ and $M=5$ tests, two versions of LAPG successfully converge using the same number of iterations as two versions of PG, but they require fewer number of communication rounds than PG. The performance gain is sizable in terms of communication. 


\section{Conclusions}
This paper studied the distributed reinforcement learning (RL) problem involving a central controller and a group of heterogeneous learners. 
Targeting applications in communication-constrained environments, we developed a communication-cognizant method for distributed RL that we term Lazily Aggregated Policy Gradient (LAPG) approach. LAPG can achieve the same convergence rates as PG, and requires fewer communication rounds given that the learners in distributed RL are heterogeneous. 
While LAPG enjoys reduced communication overhead, it requires sufficiently many mini-batch trajectories to obtain low-variance policy gradients. Therefore, our future work will aim at reducing the sample complexity by properly reusing outdated trajectories via off-policy methods and extending the algorithms to the decentralized settings.




\appendix  
\subsection{Preliminary lemmas}
Define the single-trajectory stochastic policy gradient as
\begin{align}\label{app-eq.POPG2}
\small
\hat{\nabla}_T{\cal L}_m(\bbtheta)=\sum_{t=0}^T\left(\sum_{\tau=0}^t\nabla\log\bbpi(\mathbf{a}_{\tau}|\mathbf{s}_{\tau};\bbtheta)\right)\gamma^t\ell_m(\mathbf{s}_t,\mathbf{a}_t).
\end{align}

We have the following lemma that bounds the discrepancy between them.
\vspace{0.1cm}
\begin{lemma}[bounded PG deviation]\label{lemma_bdPG}
For the finite-horizon approximation of the policy gradient \eqref{app-eq.POPG} and its corresponding version \eqref{app-eq.POPG2}, at any $\bbtheta$ and any learner $m$, their discrepancy is bounded by
\begin{equation}
\Big\|\hat{\nabla}_{T}{\cal L}_m\big(\bbtheta\big)-\nabla_T{\cal L}_m\big(\bbtheta\big)\Big\|\leq V_m
\end{equation}
where $V_m$ is a constant depending on $G, \gamma, \bar{\ell}_m$.
\end{lemma}
\begin{IEEEproof}
Using the definition of the G(PO)MDP gradient \eqref{app-eq.PG2}, we have that
\begin{align}
&\small\Big\|\hat{\nabla}_T{\cal L}_m\big(\bbtheta\big)-\nabla_T{\cal L}_m\big(\bbtheta\big)\Big\|\nonumber\\
=&\small\Bigg\|\sum_{t=0}^T\left(\sum_{\tau=0}^t\nabla\log\bbpi(\mathbf{a}_{\tau}|\mathbf{s}_{\tau};\bbtheta)\right)\gamma^t\ell_m(\mathbf{s}_t,\mathbf{a}_t)\nonumber\\
&\small-\mathbb{E}_{{\cal T}\sim\mathbb{P}(\cdot|\bbtheta)}\left[\sum_{t=0}^T\left(\sum_{\tau=0}^t\nabla\log\bbpi(\mathbf{a}_{\tau}|\mathbf{s}_{\tau};\bbtheta)\right)\gamma^t\ell_m(\mathbf{s}_t,\mathbf{a}_t)\right]\Bigg\|\nonumber\\
\leq &\small2\sup_{{\cal T}\sim\mathbb{P}(\cdot|\bbtheta)}\sum_{t=0}^T\Bigg\|\left(\sum_{\tau=0}^t\nabla\log\bbpi(\mathbf{a}_{\tau}|\mathbf{s}_{\tau};\bbtheta)\right)\gamma^t\ell_m(\mathbf{s}_t,\mathbf{a}_t)\Bigg\|\nonumber\\
\stackrel{(a)}{\leq}&\small 2\sum_{t=0}^T t G\gamma^t \bar{\ell}_m\leq 2G\bar{\ell}_m \sum_{t=0}^{\infty} t\gamma^t=\small\frac{2G\bar{\ell}_m \gamma}{(1-\gamma)^2} :=V_m
\end{align}
where (a) follows from the upper bounds in Assumptions 1 and 2, and $V_m$ is the uniform upper bound of the G(PO)MDP stochastic policy gradient.
\end{IEEEproof}

\begin{lemma}[finite horizon approximation]\label{lemma-finite}
For the infinite-horizon problem \eqref{opt0} and its finite-horizon approximation, for any $\bbtheta$, the corresponding policy gradients are bounded by
\begin{equation}
\left\|\nabla {\cal L}(\bbtheta)-\nabla_T {\cal L}(\bbtheta)\right\|\leq \sum_{m\in{\cal M}}G\bar{\ell}_m\left(T+\frac{\gamma}{1-\gamma}\right)\gamma^T:=\sigma_T.	
\end{equation}
\end{lemma}
\begin{IEEEproof}
	For any $\bbtheta\in \mathbb{R}^d$, it follows that
\begin{align}\label{eq.lemma-app-1}
&~~~\small\left\|\nabla {\cal L}(\bbtheta)-\nabla_T {\cal L}(\bbtheta)\right\|\nonumber\\
&\small=\left\|\mathbb{E}_{{\cal T}\sim\mathbb{P}(\cdot|\bbtheta)}\left[\sum_{m\in{\cal M}}\sum_{t=T}^{\infty}\left(\sum_{\tau=0}^t\nabla\log\bbpi(\mathbf{a}_{\tau}|\mathbf{s}_{\tau};\bbtheta)\right)\gamma^t\ell_m(\mathbf{s}_t,\mathbf{a}_t)\right]\right\|\nonumber\\
	&\small\stackrel{(a)}{\leq} \mathbb{E}_{{\cal T}\sim\mathbb{P}(\cdot|\bbtheta)}\left[\left\|\sum_{m\in{\cal M}}\sum_{t=T}^{\infty}\left(\sum_{\tau=0}^t\nabla\log\bbpi(\mathbf{a}_{\tau}|\mathbf{s}_{\tau};\bbtheta)\right)\gamma^t\ell_m(\mathbf{s}_t,\mathbf{a}_t)\right\|\right]\nonumber\\
	&\small\stackrel{(b)}{\leq} \mathbb{E}_{{\cal T}\sim\mathbb{P}(\cdot|\bbtheta)}\left[\sum_{m\in{\cal M}}\sum_{t=T}^{\infty}\left\|\left(\sum_{\tau=0}^t\nabla\log\bbpi(\mathbf{a}_{\tau}|\mathbf{s}_{\tau};\bbtheta)\right)\gamma^t\ell_m(\mathbf{s}_t,\mathbf{a}_t)\right\|\right]\nonumber\\
	&\small\stackrel{(c)}{\leq} \mathbb{E}_{{\cal T}\sim\mathbb{P}(\cdot|\bbtheta)}\left[\sum_{m\in{\cal M}}\sum_{t=T}^{\infty}tG\gamma^t\bar{\ell}_m\right]\nonumber\\
&\small	=\!\mathbb{E}_{{\cal T}\sim\mathbb{P}(\cdot|\bbtheta)}\left[\sum_{m\in{\cal M}}G\bar{\ell}_m\!\sum_{t=T}^{\infty}t\gamma^t\right]\!
\end{align}
where (a) uses the Jensen's inequality, (b) follows from the triangular inequality, and (c) uses the bounds on the loss and the score functions in Assumptions 1 and 2. 
We can calculate the summation as
\begin{equation}\label{eq.lemma-app-2}
	\sum_{t=T}^{\infty}t\gamma^t=\left(\frac{T}{1-\gamma}+\frac{\gamma}{(1-\gamma)^2}\right)\gamma^T.
\end{equation}
Plugging \eqref{eq.lemma-app-2} into \eqref{eq.lemma-app-1} leads to
\begin{align}\label{eq.lemma-app-3}
\small\|\nabla {\cal L}(\bbtheta)-\nabla_T {\cal L}(\bbtheta)\|&\small\leq \mathbb{E}_{{\cal T}\sim\mathbb{P}(\cdot|\bbtheta)}\left[\sum_{m\in{\cal M}}G\bar{\ell}_m\!\sum_{t=T}^{\infty}t\gamma^t\right]\nonumber\\
&\small=\sum_{m\in{\cal M}}G\bar{\ell}_m\left(T+\frac{\gamma}{1-\gamma}\right)\frac{\gamma^T}{1-\gamma}
\end{align}
from which the proof is complete.
\end{IEEEproof}

\subsection{Proof of Lemma \ref{lemma1}}\label{pf.Lemma1}
	Using the smoothness of ${\cal L}$ in Lemma \ref{lemma-smooth}, we have that
\begin{equation}\label{eq.pf-lemma1-1}
\small
	{\cal L}(\bbtheta^{k+1})-{\cal L}(\bbtheta^k)\leq \left\langle \nabla{\cal L}(\bbtheta^k), \bbtheta^{k+1}-\bbtheta^k\right\rangle+\frac{L}{2}\left\|\bbtheta^{k+1}-\bbtheta^k\right\|^2.
\end{equation}

Note that \eqref{eq.LAG2} can be also written as (cf. {\small$\hat{\nabla}_{N,T}{\cal L}\big(\bbtheta^k\big):=\sum_{m\in{\cal M}}\hat{\nabla}_{N,T}{\cal L}_m\big(\bbtheta^k\big)$})
	\begin{align}\label{eq.LAG3}
	\bbtheta^{k+1}=\bbtheta^k-\alpha \hat{\nabla}_{N,T}{\cal L}\big(\bbtheta^k\big)+\alpha\sum_{m\in{{\cal M}^k_c}}\delta\hat{\nabla}^k_m
	\end{align}
where ${\cal M}^k_c$ is the set of agents that \emph{do not} communicate with the controller at iteration $k$.

Plugging \eqref{eq.LAG3} into $\langle \nabla{\cal L}(\bbtheta^k), \bbtheta^{k+1}-\bbtheta^k\rangle$ leads to 
\begin{align}\label{eq.pf-lemma1-3}
&\small \Big\langle \nabla{\cal L}(\bbtheta^k), \bbtheta^{k+1}-\bbtheta^k\Big\rangle\\
=&\small -\alpha\left\langle \nabla{\cal L}(\bbtheta^k),\hat{\nabla}_{N,T}{\cal L}\big(\bbtheta^k\big)-\sum_{m\in{\cal M}^k_c}\delta\hat{\nabla}^k_m\right\rangle\nonumber\\
=&\small -\alpha\left\langle \nabla{\cal L}(\bbtheta^k),\nabla{\cal L}(\bbtheta^k)-\nabla{\cal L}(\bbtheta^k)+\hat{\nabla}_{N,T}{\cal L}\big(\bbtheta^k\big)\!-\!\!\!\sum_{m\in{\cal M}^k_c}\!\!\delta\hat{\nabla}^k_m\right\rangle\nonumber\\
= &\small-\alpha\left\|\nabla{\cal L}(\bbtheta^k)\right\|^2\nonumber\\
& -\alpha\left\langle\! \nabla{\cal L}(\bbtheta^k),\hat{\nabla}_{N,T}{\cal L}\big(\bbtheta^k\big)-\nabla{\cal L}(\bbtheta^k)-\!\sum_{m\in{\cal M}^k_c}\!\delta\hat{\nabla}^k_m\right\rangle.\nonumber
\end{align}

Using $2\mathbf{a}^{\top}\mathbf{b}=\|\mathbf{a}\|^2+\|\mathbf{b}\|^2-\|\mathbf{a}-\mathbf{b}\|^2$, we can re-write the inner product in \eqref{eq.pf-lemma1-3} as
\begin{align}\label{eq.pf-lemma1-4}
&\small \left\langle -\nabla{\cal L}(\bbtheta^k),\hat{\nabla}_{N,T}{\cal L}\big(\bbtheta^k\big)-\nabla{\cal L}(\bbtheta^k)-\sum_{m\in{\cal M}^k_c}\delta\hat{\nabla}^k_m\right\rangle\nonumber\\
=&\small \frac{1}{2}\left\|\nabla{\cal L}(\bbtheta^k)\right\|^2+\frac{1}{2}\Bigg\|\hat{\nabla}_{N,T}{\cal L}\big(\bbtheta^k\big)-\nabla{\cal L}(\bbtheta^k)-\sum_{m\in{\cal M}^k_c}\delta\hat{\nabla}^k_m\Bigg\|^2\nonumber\\
&\small -\frac{1}{2}\Bigg\|\hat{\nabla}_{N,T}{\cal L}\big(\bbtheta^k\big)-\sum_{m\in{\cal M}^k_c}\delta\hat{\nabla}^k_m\Bigg\|^2\nonumber\\
\stackrel{(a)}{=}&\small \frac{1}{2}\left\|\nabla{\cal L}(\bbtheta^k)\right\|^2\!+\frac{1}{2}\Bigg\|\hat{\nabla}_{N,T}{\cal L}\big(\bbtheta^k\big)-\nabla{\cal L}(\bbtheta^k)-\sum_{m\in{\cal M}^k_c}\delta\hat{\nabla}^k_m\Bigg\|^2\nonumber\\
&\small -\frac{1}{2\alpha^2}\left\|\bbtheta^{k+1}-\bbtheta^k\right\|^2
\end{align}
where (a) follows from the LAPG update \eqref{eq.LAG3}.

Define the policy gradient for the finite-horizon discounted reward as 
\begin{align}\label{eq.PG-h}
&\small \nabla_T {\cal L}(\bbtheta):=\!\!\sum_{m\in{\cal M}}\!\!\nabla_T {\cal L}_m(\bbtheta)\\
&\small {\rm with}~~\nabla_T {\cal L}_m(\bbtheta)=\mathbb{E}_{{\cal T}\sim\mathbb{P}(\cdot|\bbtheta)}\left[\nabla\log\mathbb{P}({\cal T}|\bbtheta)\left(\sum_{t=0}^T\gamma^t\ell_m(\mathbf{s}_t,\mathbf{a}_t)\right)\right]\nonumber
\end{align}
and decompose the second term in \eqref{eq.pf-lemma1-4} as 
\begin{align}\label{eq.pf-lemma1-5}
&\small\Big\|\hat{\nabla}_{N,T}{\cal L}\big(\bbtheta^k\big)-\nabla{\cal L}(\bbtheta^k)-\sum_{m\in{\cal M}^k_c}\delta\hat{\nabla}^k_m\Big\|^2\nonumber\\
=&\small\Big\|\hat{\nabla}_{N,T}{\cal L}\big(\bbtheta^k\big)-\nabla_T {\cal L}(\bbtheta^k)+\nabla_T {\cal L}(\bbtheta^k)-\nabla{\cal L}(\bbtheta^k)-\!\!\!\sum_{m\in{\cal M}^k_c}\!\!\delta\hat{\nabla}^k_m\Big\|^2\nonumber\\
\stackrel{(b)}{=}&3\Big\|\hat{\nabla}_{N,T}{\cal L}\big(\bbtheta^k\big)-\nabla_T {\cal L}(\bbtheta^k)\Big\|^2+3\Big\|\sum_{m\in{\cal M}^k_c}\delta\hat{\nabla}^k_m\Big\|^2\nonumber\\
&\small+3\Big\|\nabla_T {\cal L}(\bbtheta^k)-\nabla{\cal L}(\bbtheta^k)\Big\|^2
\end{align}
where (b) follows from the inequality $\|\mathbf{a}+\mathbf{b}+\mathbf{c}\|^2\leq 3\|\mathbf{a}\|^2+3\|\mathbf{b}\|^2+3\|\mathbf{c}\|^2$.
Combining \eqref{eq.pf-lemma1-3}, \eqref{eq.pf-lemma1-4} and \eqref{eq.pf-lemma1-5}, and plugging into \eqref{eq.pf-lemma1-1}, the claim of Lemma \ref{lemma1} follows.

\subsection{Proof of Lemma \ref{lemma.PGwhp}}\label{pf.lemma2}
The policy gradient concentration result in Lemma \ref{lemma.PGwhp} builds on the following concentration inequality.
\begin{lemma}[concentration inequality \cite{pinelis1994}]\label{lemma.whp}
If $\mathbf{X}_1, \mathbf{X}_2, \cdots, \mathbf{X}_N\in \mathbb{R}^d$ denote a vector-valued martingale difference sequence satisfying $\mathbb{E}[\mathbf{X}_n|\mathbf{X}_1,\cdots,\mathbf{X}_{n-1}]=\mathbf{0}$, and $\|\mathbf{X}_n\|\leq V,\,\forall n$, then for any scalar $\delta\in(0,1]$, we have 
\begin{equation}
\small	\mathbb{P}\left(\left\|\sum_{n=1}^N\mathbf{X}_n\right\|^2>2\log (2/\delta)V^2N \right)\leq \delta.
\end{equation} 
\end{lemma}

Therefore, viewing $\mathbf{X}_n:=\hat{\nabla}_T{\cal L}_m\big(\bbtheta\big)-\nabla_T{\cal L}_m\big(\bbtheta\big)$, and using the bounded PG deviation in Lemma \ref{lemma_bdPG}, we can readily arrive at Lemma \ref{lemma.PGwhp}.

\subsection{Proof of Lemma \ref{lemma-smooth}}\label{pf.lemma3}
For notational brevity, we use ${\cal T}\sim\bbtheta_1$ for ${\cal T}\sim\mathbb{P}(\cdot|\bbtheta_1)$, and $\bbpi(\mathbf{a}_t;\bbtheta)$ for $\bbpi(\mathbf{a}_t|\mathbf{s}_t;\bbtheta)$ in this proof.
	For any $\bbtheta_1, \bbtheta_2\in \mathbb{R}^d$, it follows that 
\begin{align}\label{eq.smooth-1}
	& \small\|\nabla {\cal L}_m(\bbtheta_1)-\nabla{\cal L}_m(\bbtheta_2)\|\nonumber\\
\!\!=& \small\Bigg\|\mathbb{E}_{{\cal T}\sim\bbtheta_1}\left[\sum_{t=0}^{\infty}\left(\sum_{\tau=0}^t\nabla\log\bbpi(\mathbf{a}_{\tau};\bbtheta_1)\right)\gamma^t\ell_m(\mathbf{s}_t,\mathbf{a}_t)\right]\nonumber\\
\!\!	-& \small\mathbb{E}_{{\cal T}\sim\bbtheta_2}\left[\sum_{t=0}^{\infty}\left(\sum_{\tau=0}^t\nabla\log\bbpi(\mathbf{a}_{\tau};\bbtheta_1)\right)\gamma^t\ell_m(\mathbf{s}_t,\mathbf{a}_t)\right]\nonumber\\
\!\!    +& \small\mathbb{E}_{{\cal T}\sim\bbtheta_2}\left[\sum_{t=0}^{\infty}\left(\sum_{\tau=0}^t\nabla\log\bbpi(\mathbf{a}_{\tau};\bbtheta_1)\right)\gamma^t\ell_m(\mathbf{s}_t,\mathbf{a}_t)\right]\nonumber\\
\!\!	-& \small\mathbb{E}_{{\cal T}\sim\bbtheta_2}\!\left[\sum_{t=0}^{\infty}\left(\sum_{\tau=0}^t\nabla\log\bbpi(\mathbf{a}_{\tau};\bbtheta_2)\right)\gamma^t\ell_m(\mathbf{s}_t,\mathbf{a}_t)\right]\!\Bigg\|.
\end{align}

We can bound the first difference term in \eqref{eq.smooth-1} as 
\begin{align}\label{eq.smooth-2}
&\small \Bigg\|\mathbb{E}_{{\cal T}\sim\bbtheta_1}\!\!\left[\sum_{t=0}^{\infty}\left(\sum_{\tau=0}^t\nabla\log\bbpi(\mathbf{a}_{\tau};\bbtheta_1)\right)\!\gamma^t\ell_m(\mathbf{s}_t,\mathbf{a}_t)\right]\nonumber\\
&\small~~-\mathbb{E}_{{\cal T}\sim\bbtheta_2}\!\!\left[\sum_{t=0}^{\infty}\left(\sum_{\tau=0}^t\nabla\log\bbpi(\mathbf{a}_{\tau};\bbtheta_1)\!\right)\gamma^t\ell_m(\mathbf{s}_t,\mathbf{a}_t)\right]\!\Bigg\|\nonumber\\
\!\!= &\small \Bigg\|\int\mathbb{P}({\cal T}|\bbtheta_1)\sum_{t=0}^{\infty}\left(\sum_{\tau=0}^t\nabla\log\bbpi(\mathbf{a}_{\tau};\bbtheta_1)\right)\!\gamma^t\ell_m(\mathbf{s}_t,\mathbf{a}_t)\nonumber\\
&~~\small-\mathbb{P}({\cal T}|\bbtheta_2)\sum_{t=0}^{\infty}\left(\sum_{\tau=0}^t\nabla\log\bbpi(\mathbf{a}_{\tau};\bbtheta_1)\!\right)\!\gamma^t\ell_m(\mathbf{s}_t,\mathbf{a}_t)\mathbf{d}{\cal T}\Bigg\|\nonumber\\
\!\!= &\small \Bigg\|\sum_{t=0}^{\infty}\int\mathbb{P}({\cal T}_t|\bbtheta_1)\left(\sum_{\tau=0}^t\nabla\log\bbpi(\mathbf{a}_{\tau};\bbtheta_1)\right)\!\gamma^t\ell_m(\mathbf{s}_t,\mathbf{a}_t)\nonumber\\
&~~~~~~\small-\mathbb{P}({\cal T}_t|\bbtheta_2)\left(\sum_{\tau=0}^t\nabla\log\bbpi(\mathbf{a}_{\tau};\bbtheta_1)\!\right)\!\gamma^t\ell_m(\mathbf{s}_t,\mathbf{a}_t)\mathbf{d}{\cal T}_t\Bigg\|\nonumber\\
\!\!\leq &\small\sum_{t=0}^{\infty}\!\int\!\Big\|\!\big(\mathbb{P}({\cal T}_t|\bbtheta_1)\!-\!\mathbb{P}({\cal T}_t|\bbtheta_2)\big)\!\Big(\!\sum_{\tau=0}^t\!\nabla\log\bbpi(\mathbf{a}_{\tau};\bbtheta_1)\!\Big)\gamma^t\ell_m(\mathbf{s}_t,\mathbf{a}_t)\Big\|\mathbf{d}{\cal T}_t\nonumber\\
\!\!\leq &\small  \sum_{t=0}^{\infty}\!\int\!\Big|\mathbb{P}({\cal T}_t|\bbtheta_1)\!-\!\mathbb{P}({\cal T}_t|\bbtheta_2)\Big|\Big\|\!\Big(\sum_{\tau=0}^t\nabla\log\bbpi(\mathbf{a}_{\tau};\bbtheta_1)\!\Big)\!\gamma^t\ell_m(\mathbf{s}_t,\mathbf{a}_t)\Big\|\mathbf{d}{\cal T}_t
\end{align}
where we use ${\cal T}_t$ to denote a $t$-slot trajectory $\{\mathbf{s}_0,\mathbf{a}_0, \mathbf{s}_1,\mathbf{a}_1,\cdots,\mathbf{s}_{t-1},\mathbf{a}_{t-1}\}$.

For the remaining difference term in \eqref{eq.smooth-2}, we bound it as
 \begin{align}\label{eq.smooth-3}
  	& \small\Big|\mathbb{P}({\cal T}_t|\bbtheta_1)-\mathbb{P}({\cal T}_t|\bbtheta_2)\Big|\nonumber\\
 	=&\small\Big|\rho(\mathbf{s}_0)\prod_{\nu=0}^{t-1} \bbpi(\mathbf{a}_{\nu};\bbtheta_1)\mathbb{P}(\mathbf{s}_{\nu+1}|\mathbf{s}_{\nu},\mathbf{a}_{\nu})\nonumber\\
 	&\small\qquad\qquad\qquad\qquad\quad~-\rho(\mathbf{s}_0)\prod_{\nu=0}^{t-1} \bbpi(\mathbf{a}_{\nu};\bbtheta_2)\mathbb{P}(\mathbf{s}_{\nu+1}|\mathbf{s}_{\nu},\mathbf{a}_{\nu})\Big|\nonumber\\
 	=&\small\rho(\mathbf{s}_0)\prod_{\nu=0}^{t-1}\mathbb{P}(\mathbf{s}_{\nu+1}|\mathbf{s}_{\nu},\mathbf{a}_{\nu})\Big|\prod_{\nu=0}^{t-1} \bbpi(\mathbf{a}_{\nu};\bbtheta_1)-\prod_{\nu=0}^{t-1} \bbpi(\mathbf{a}_{\nu};\bbtheta_2)\Big|\nonumber\\
 	\stackrel{(a)}{=}&\small\rho(\mathbf{s}_0)\prod_{\nu=0}^{t-1}\mathbb{P}(\mathbf{s}_{\nu+1}|\mathbf{s}_{\nu},\mathbf{a}_{\nu})\Big|(\bbtheta_1-\bbtheta_2)^{\top}\nabla\prod_{\nu=0}^{t-1}\bbpi(\mathbf{a}_{\nu};\tilde{\bbtheta})\Big|
 \end{align}
 where (a) uses the mean-value theorem and $\tilde{\bbtheta}:=(1-c)\bbtheta_1+c\bbtheta_2$ with a certain constant $c\in[0, 1]$. 
 
 In addition, note that we have
 \begin{align}\label{eq.smooth-4}
 & \small\Big|(\bbtheta_1-\bbtheta_2)^{\top}\nabla\prod_{\nu=0}^{t-1}\bbpi(\mathbf{a}_{\nu};\tilde{\bbtheta})\Big|\nonumber\\
 	=	&\small\Big|(\bbtheta_1-\bbtheta_2)^{\top}\nabla\log\prod_{\nu=0}^{t-1}\bbpi(\mathbf{a}_{\nu};\tilde{\bbtheta})\prod_{\nu=0}^{t-1}\bbpi(\mathbf{a}_{\nu};\tilde{\bbtheta})\Big|\nonumber\\
 	=&\small\prod_{\nu=0}^{t-1}\bbpi(\mathbf{a}_{\nu};\tilde{\bbtheta})\Big|(\bbtheta_1-\bbtheta_2)^{\top}\nabla\log\prod_{\nu=0}^{t-1}\bbpi(\mathbf{a}_{\nu};\tilde{\bbtheta})\Big|\nonumber\\
 	\leq &\small\prod_{\nu=0}^{t-1}\bbpi(\mathbf{a}_{\nu};\tilde{\bbtheta})\Big\|\sum_{\nu=0}^{t-1}\nabla\log\bbpi(\mathbf{a}_{\nu};\tilde{\bbtheta})\Big\|\Big\|\bbtheta_1-\bbtheta_2\Big\|\nonumber\\
 	\leq &\small \,tG\Big\|\bbtheta_1-\bbtheta_2\Big\|
\prod_{\nu=0}^{t-1}\bbpi(\mathbf{a}_{\nu};\tilde{\bbtheta}) 
\end{align}
and if plugging \eqref{eq.smooth-3} and \eqref{eq.smooth-4} into \eqref{eq.smooth-2}, it follows that
 \begin{align}\label{eq.smooth-5}
 & \small \sum_{t=0}^{\infty}\!\int\!\Big|\mathbb{P}({\cal T}_t|\bbtheta_1)-\mathbb{P}({\cal T}_t|\bbtheta_2)\Big|\Big\|\sum_{\tau=0}^t\nabla\log\bbpi(\mathbf{a}_{\tau};\bbtheta_1)\!\gamma^t\ell_m(\mathbf{s}_t,\mathbf{a}_t)\Big\|\mathbf{d}{\cal T}_t\nonumber\\
& \leq  \small \sum_{t=0}^{\infty}\int\rho(\mathbf{s}_0)\prod_{\nu=0}^{t-1}\mathbb{P}(\mathbf{s}_{\nu+1}|\mathbf{s}_{\nu},\mathbf{a}_{\nu})\bbpi(\mathbf{a}_{\nu};\tilde{\bbtheta}) tG\Big\|\bbtheta_1-\bbtheta_2\Big\|\nonumber\\
& \small \qquad\qquad\qquad\quad\times \Bigg\|\left(\sum_{\tau=0}^t\nabla\log\bbpi(\mathbf{a}_{\tau};\bbtheta_1)\right)\!\gamma^t\ell_m(\mathbf{s}_t,\mathbf{a}_t)\Bigg\|\mathbf{d}{\cal T}_t\nonumber\\
&\leq \small \sum_{t=0}^{\infty}\int	\mathbb{P}({\cal T}_t|\tilde{\bbtheta})tG\Big\|\bbtheta_1-\bbtheta_2\Big\| tG\gamma^t\bar{\ell}_m\mathbf{d}{\cal T}_t\nonumber\\
&= \small\sum_{t=0}^{\infty}	t^2G^2\gamma^t\bar{\ell}_m\Big\|\bbtheta_1-\bbtheta_2\Big\| \nonumber\\
&= \small\left(\frac{\gamma}{(1-\gamma)^2}+\frac{2\gamma^2}{(1-\gamma)^3}\right)G^2\bar{\ell}_m\Big\|\bbtheta_1-\bbtheta_2\Big\|
 \end{align}
 where we use the equation that $\sum_{t=0}^{\infty}t^2\gamma^t=\frac{\gamma}{(1-\gamma)^2}+\frac{2\gamma^2}{(1-\gamma)^3}$.

We can separately bound the second difference term in \eqref{eq.smooth-1} as
\begin{align}\label{eq.smooth-6}
&~~\small\Bigg\|\mathbb{E}_{{\cal T}\sim\bbtheta_2}\!\!\left[\sum_{t=0}^{\infty}\left(\sum_{\tau=0}^t\nabla\log\bbpi(\mathbf{a}_t;\bbtheta_1)\right)\!\gamma^t\ell_m(\mathbf{s}_t,\mathbf{a}_t)\right]\nonumber\\
& \small~~~~~-\mathbb{E}_{{\cal T}\sim\bbtheta_2}\!\!\left[\sum_{t=0}^{\infty}\left(\sum_{\tau=0}^t\nabla\log\bbpi(\mathbf{a}_t;\bbtheta_2)\!\right)\gamma^t\ell_m(\mathbf{s}_t,\mathbf{a}_t)\right]\!\Bigg\|\nonumber\\
\!\!\leq & \small\int\mathbb{P}({\cal T}|\bbtheta_2)\Bigg\|\sum_{t=0}^{\infty}\left(\sum_{\tau=0}^t\nabla\log\bbpi(\mathbf{a}_t;\bbtheta_1)\right)\!\gamma^t\ell_m(\mathbf{s}_t,\mathbf{a}_t)\nonumber\\
&\small\qquad\qquad\qquad-\sum_{t=0}^{\infty}\left(\sum_{\tau=0}^t\nabla\log\bbpi(\mathbf{a}_t;\bbtheta_2)\!\right)\!\gamma^t\ell_m(\mathbf{s}_t,\mathbf{a}_t)\Bigg\|\mathbf{d}{\cal T}\nonumber\\
\!\!\leq & \small\!\int\!\mathbb{P}({\cal T}|\bbtheta_2)\!\sum_{t=0}^{\infty}\!\gamma^t\ell_m(\mathbf{s}_t,\mathbf{a}_t)\nonumber\\
&\qquad\qquad\quad\times\sum_{\tau=0}^t\!\Big\|\nabla\log\bbpi(\mathbf{a}_t;\bbtheta_1)\!-\!\nabla\log\bbpi(\mathbf{a}_t;\bbtheta_2)\Big\|\mathbf{d}{\cal T}\nonumber\\
\!\!\leq & \small \int\mathbb{P}({\cal T}|\bbtheta_2)\sum_{t=0}^{\infty}\gamma^t\bar{\ell}_m\sum_{\tau=0}^tF\big\|\bbtheta_1-\bbtheta_2\big\|\mathbf{d}{\cal T}\nonumber\\
\!\!\leq & \small \sum_{t=0}^{\infty}\gamma^t\bar{\ell}_m tF\big\|\bbtheta_1-\bbtheta_2\big\|=\frac{F\bar{\ell}_m\gamma}{(1-\gamma)^2}\big\|\bbtheta_1-\bbtheta_2\big\|.
\end{align}
Combining \eqref{eq.smooth-5} and \eqref{eq.smooth-6}, we have that
\begin{align}
&\small\|\nabla {\cal L}_m(\bbtheta_1)-\nabla{\cal L}_m(\bbtheta_2)\|\nonumber\\
\leq&\small\left(\frac{F}{(1-\gamma)^2}+\left(\frac{1}{(1-\gamma)^2}+\frac{2\gamma}{(1-\gamma)^3}\right)G^2\right)\gamma\bar{\ell}_m\big\|\bbtheta_1-\bbtheta_2\big\|\nonumber\\
:=&\small L_m\big\|\bbtheta_1-\bbtheta_2\big\|.
\end{align}
Similarly, we can bound the Lipschitz constant of $\nabla {\cal L}(\bbtheta)$, $\nabla_T {\cal L}(\bbtheta)$, $\nabla_T{\cal L}_m(\bbtheta)$, and the proof is complete.

\subsection{Proof of Theorem \ref{theorem1}}\label{pf.theorem1}
The subsequent analysis critically builds on the following Lyapunov function:
 \begin{equation}\label{eq.Lyap}
 \small
 	\mathbb{V}^k:={\cal L}(\bbtheta^k)-{\cal L}(\bbtheta^*)+ \frac{3}{2\alpha}\sum_{d=1}^D\sum_{\tau=d}^D \xi\left\|\bbtheta^{k+1-d}-\bbtheta^{k-d}\right\|^2
 \end{equation}
where $\bbtheta^*$ is the minimizer of \eqref{opt0}, and $\alpha, \xi$ are the stepsize and the threshold constant in \eqref{eq.trig-cond}.
Using the definition of $\mathbb{V}^k$ in \eqref{eq.Lyap}, it follows that (with $\beta_d:=\frac{3}{2\alpha}\sum_{\tau=d}^D \xi$)
\begin{align}\label{eq.pf-lemma2-1}
&\mathbb{V}^{k+1}-\mathbb{V}^k\nonumber\\
=&\small{\cal L}(\bbtheta^{k+1})-{\cal L}(\bbtheta^k)+\sum_{d=1}^D \beta_d\left\|\bbtheta^{k+2-d}-\bbtheta^{k+1-d}\right\|^2\nonumber\\
-&\small\sum_{d=1}^D \beta_d\left\|\bbtheta^{k+1-d}-\bbtheta^{k-d}\right\|^2\nonumber\\
\!\!\stackrel{(a)}{\leq}&\small\!-\!\frac{\alpha}{2}\left\|\nabla{\cal L}(\bbtheta^k)\right\|^2\!\!+\!\frac{3\alpha}{2}\Bigg\|\!\sum_{m\in{\cal M}^k_c}\!\!\!\delta\hat{\nabla}^k_m\Bigg\|^2\!\!+\!\frac{3\alpha}{2}\Big\|\nabla_T {\cal L}(\bbtheta^k)\!-\!\nabla{\cal L}(\bbtheta^k)\Big\|^2\nonumber\\
\!\!+&\small\frac{3\alpha}{2}\Big\|\hat{\nabla}_{N,T}{\cal L}\big(\bbtheta^k\big)\!-\!\nabla_T {\cal L}(\bbtheta^k)\Big\|^2+\left(\!\frac{L}{2}-\frac{1}{2\alpha}+\beta_1\!\right)\!\left\|\bbtheta^{k+1}\!-\!\bbtheta^k\right\|^2\!\nonumber\\
\!\!+&\small\sum_{d=2}^D \beta_d\left\|\bbtheta^{k+2-d}\!-\!\bbtheta^{k+1-d}\right\|^2\!-\!\sum_{d=1}^D \beta_d\left\|\bbtheta^{k+1-d}\!-\!\bbtheta^{k-d}\right\|^2\!\!\!
\end{align}
where (a) uses \eqref{eq.lemma1} in Lemma \ref{lemma1}.

Using $(\sum_{n=1}^N a_n)^2\leq N\sum_{n=1}^N a_n^2$, it follows that
	\begin{align}\label{eq.pf-lemma2-4}
\Bigg\|\sum_{m\in{\cal M}^k_c}\delta\hat{\nabla}^k_m\Bigg\|^2\!\!=\,& \small\Bigg\|\sum_{m\in{\cal M}^k_c}\nabla{\cal L}_m\big(\hat{\bbtheta}_m^k\big)-\nabla{\cal L}_m\big(\bbtheta^k\big)\Bigg\|^2\nonumber\\
\leq & \small\left|{\cal M}^k_c\right|\sum_{m\in{\cal M}^k_c}\Big\|\nabla{\cal L}_m\big(\hat{\bbtheta}_m^k\big)-\nabla{\cal L}_m\big(\bbtheta^k\big)\Big\|^2\nonumber\\
\stackrel{(b)}{\leq}& \small\frac{|{\cal M}^k_c|^2}{\alpha^2M^2}D \xi\left\|\bbtheta^{k+1-d}-\bbtheta^{k-d}\right\|^2+6\sigma^2_{N,\delta/K}
\end{align}
where (b) uses the communication trigger condition \eqref{eq.trig-cond}, and the fact that $\sigma^2_{N,\delta/K}:=M\sum_{m\in{\cal M}}\sigma^2_{m,N,\delta/K}$.

Plugging \eqref{eq.pf-lemma2-4} into \eqref{eq.pf-lemma2-1}, we have (for convenience, define $\beta_{D+1}=0$ in the analysis)
\begin{align}\label{eq.pf-lemma2-5}
&\small\mathbb{V}^{k+1}-\mathbb{V}^k \nonumber\\
\!\!\leq &\small\!-\!\frac{\alpha}{2}\left\|\nabla{\cal L}(\bbtheta^k)\right\|^2\!\!+\!\sum_{d=1}^D\Big(\frac{3 \xi\left|{\cal M}^k_c\right|^2}{2\alpha M^2}-\beta_d+\beta_{d+1}\!\Big)\!\left\|\bbtheta^{k+1-d}\!\!-\!\bbtheta^{k-d}\right\|^2\nonumber\\
\!\!+&\small\left(\!\frac{L}{2}-\frac{1}{2\alpha}+\beta_1\!\right)\!\left\|\bbtheta^{k+1}\!-\!\bbtheta^k\right\|^2+\frac{3\alpha}{2}\Big\|\nabla_T {\cal L}(\bbtheta^k)\!-\!\nabla{\cal L}(\bbtheta^k)\Big\|^2\nonumber\\
\!\!+&\frac{3\alpha}{2}\Big\|\hat{\nabla}_{N,T}{\cal L}\big(\bbtheta^k\big)\!-\!\nabla_T {\cal L}(\bbtheta^k)\Big\|^2+9\alpha\sigma^2_{N,\delta/K}.
\end{align}
After defining some constants to simplify the notation, the proof is then complete.

Furthermore, using $\beta_d:=\frac{3}{2\alpha}\sum_{\tau=d}^D \xi$, if the stepsize $\alpha$, and the trigger constants $\{ \xi\}$ satisfy
	\begin{align}\label{eq.beta}
\alpha \leq \frac{\big(1-3D\xi\big)}{L}
\end{align}
then it is easy to verify that expressions in the parentheses in \eqref{eq.pf-lemma2-5} are all nonpositive. Hence, we have that the descent in the Lyapunov function is bounded as 
\begin{align}\label{eq.pf-lemma2-6}
\small
\mathbb{V}^{k+1}-\mathbb{V}^k\!& \leq -\frac{\alpha}{2}\left\|\nabla{\cal L}(\bbtheta^k)\right\|^2+\frac{3\alpha}{2}\Big\|\nabla_T {\cal L}(\bbtheta^k)-\nabla{\cal L}(\bbtheta^k)\Big\|^2\nonumber\\
&+\frac{3\alpha}{2}\Big\|\hat{\nabla}_{N,T}{\cal L}\big(\bbtheta^k\big)-\nabla_T {\cal L}(\bbtheta^k)\Big\|^2+9\alpha\sigma^2_{N,\delta/K}.
\end{align}

Rearranging terms in \eqref{eq.pf-lemma2-6}, and summing up over $k=1,\cdots, K$, we have 
\begin{align}\label{eq.pf-lemma2-7}
\small\frac{1}{K}\sum_{k=1}^K\left\|\nabla{\cal L}(\bbtheta^k)\right\|^2\!\!
& \small\stackrel{(c)}{\leq} \frac{2}{\alpha K}\mathbb{V}^1\!+\!\frac{3}{K}\sum_{k=1}^K\Big\|\nabla_T {\cal L}(\bbtheta^k)\!-\!\nabla{\cal L}(\bbtheta^k)\Big\|^2\nonumber\\
& \small+\frac{3}{K}\sum_{k=1}^K\Big\|\hat{\nabla}_{N,T}{\cal L}\big(\bbtheta^k\big)\!-\!\nabla_T {\cal L}(\bbtheta^k)\Big\|^2\!\!\!+\!18\sigma^2_{N,\delta/K}\nonumber\\
& \small\stackrel{(d)}{\leq} \frac{2}{\alpha K}\mathbb{V}^1\!+3\sigma_T^2+21\sigma^2_{N,\delta/K},~~~{\rm w.p.}~1-\delta
\end{align}
where (c) omits the negative term $-\mathbb{V}^{K+1}$, and (d) follows from the finite-horizon truncation error in Lemma \ref{lemma-finite}, and the gradient concentration result in Lemma \ref{lemma.PGwhp} together with the union bound. 

Therefore, using Lemmas \ref{lemma.PGwhp} and \ref{lemma-finite}, it readily follows that there exist $T={\cal O}(\log(1/\epsilon))$, $K={\cal O}(1/\epsilon)$, and $N={\cal O} \left(\log(K/\delta)/\epsilon\right)$ such that
\begin{equation}\label{eq.pf-lemma2-8}
\small
\frac{1}{K}\sum_{k=1}^K\left\|\nabla{\cal L}(\bbtheta^k)\right\|^2
\leq\frac{2}{\alpha K}\mathbb{V}^1\!+3\sigma_T^2+21\sigma^2_{N,\delta/K}\leq \epsilon,~~~{\rm w.p.}~1-\delta
\end{equation} 
from which the proof is complete.

\subsection{Proof of Lemma \ref{lemma4}}\label{pf.lemma4}
The idea is essentially to show that if \eqref{eq.lemma4} holds, then the learner $m$ will not violate the LAPG conditions in \eqref{eq.trig-cond} so that it does not upload, if it has uploaded \emph{at least} once during the last $d$ iterations.

To prove this argument, for the difference of two policy gradient evaluations, we have that 
\begin{align}\label{eq.lemma4-1}
&\small\Big\|\hat{\nabla}_{N,T} {\cal L}_m(\hat{\bbtheta}_m^{k-1})-\hat{\nabla}_{N,T} {\cal L}_m(\bbtheta^k)\Big\|^2\nonumber\\
=&\small\Big\|\hat{\nabla}_{N,T} {\cal L}_m(\hat{\bbtheta}_m^{k-1})-\nabla_T {\cal L}_m(\hat{\bbtheta}_m^{k-1})\nonumber\\
&\small+\nabla_T{\cal L}_m(\hat{\bbtheta}_m^{k-1})-\nabla_T {\cal L}_m(\bbtheta^k)+\nabla_T {\cal L}_m(\bbtheta^k)-\hat{\nabla}_{N,T} {\cal L}_m(\bbtheta^k)\Big\|^2\nonumber\\
\stackrel{(a)}{\leq} &\small3\Big\|\hat{\nabla}_{N,T} {\cal L}_m(\hat{\bbtheta}_m^{k-1})\!-\!\nabla_T {\cal L}_m(\hat{\bbtheta}_m^{k-1})\Big\|^2\nonumber\\
\small+3&\small\Big\|\nabla_T{\cal L}_m(\hat{\bbtheta}_m^{k-1})\!-\!\nabla_T {\cal L}_m(\bbtheta^k)\Big\|^2\!\!\!+\!3\Big\|\nabla_T {\cal L}_m(\bbtheta^k)\!-\!\hat{\nabla}_{N,T} {\cal L}_m(\bbtheta^k)\Big\|^2\nonumber\\
\stackrel{(b)}{\leq} &\small 6\sigma^2_{m,N,\delta}+3\Big\|\nabla_T{\cal L}_m(\hat{\bbtheta}_m^{k-1})\!-\!\nabla_T {\cal L}_m(\bbtheta^k)\Big\|^2,\,~~~{\rm w.p.}~1-2\delta/K\nonumber\\
\stackrel{(c)}{\leq} &\small 6\sigma^2_{m,N,\delta}+3L_m^2\left\|\hat{\bbtheta}_m^{k-1}-\bbtheta^k\right\|^2,\,~~~{\rm w.p.}~1-2\delta/K
\end{align}
where (a) uses $\|\mathbf{a}+\mathbf{b}+\mathbf{c}\|^2\leq 3\|\mathbf{a}\|^2+3\|\mathbf{b}\|^2+3\|\mathbf{c}\|^2$; (b) uses Lemma \ref{lemma.PGwhp} twice; and (c) follows from the smoothness property in Lemma \ref{lemma-smooth}.

Furthermore, suppose that at iteration $k$, the most recent iteration that the learner $m$ did communicate with the controller is iteration $k-d'$ with $1\leq d'\leq d$.
Thus, we have $\hat{\bbtheta}_m^{k-1}=\bbtheta^{k-d'}$, which implies that
\begin{align}\label{eq.lemma4-2}
	&\small6\sigma^2_{m,N,\delta}+3L_m^2\left\|\hat{\bbtheta}_m^{k-1}-\bbtheta^k\right\|^2\nonumber\\
	= \, &\small6\sigma^2_{m,N,\delta}+3L_m^2\left\|\bbtheta^{k-d'}-\bbtheta^k\right\|^2\nonumber\\
\leq \, &\small6\sigma^2_{m,N,\delta}+3d'L^2\mathds{H}(m)\sum_{b=1}^{d'}\left\|\bbtheta^{k+1-b}-\bbtheta^{k-b}\right\|^2\nonumber\\
\stackrel{(d)}{\leq}\,& \small6\sigma^2_{m,N,\delta}+\frac{ \xi}{\alpha^2M^2}\sum_{b=1}^{d'}\left\|\bbtheta^{k+1-b}-\bbtheta^{k-b}\right\|^2\nonumber\\
\stackrel{(e)}{\leq}\, &\small 6\sigma^2_{m,N,\delta}+ \frac{\sum_{b=1}^D\xi_b\left\|\bbtheta^{k+1-b}-\bbtheta^{k-b}\right\|^2}{\alpha^2M^2}
\end{align}
where (d) follows since the condition \eqref{eq.lemma4} is satisfied, so that
\begin{equation}
	\mathds{H}(m)\leq \frac{ \xi}{3d\alpha^2 L^2 M^2}\leq  \frac{ \xi}{3d'\alpha^2 L^2 M^2}
\end{equation}
and (e) follows from our choice of $\{ \xi\}$ such that for $1\leq d'\leq d$, we have $ \xi\leq \xi_{d'}\leq \ldots \leq \xi_1$ and $\|\bbtheta^{k+1-b}-\bbtheta^{k-b}\|^2\geq 0$.
Since \eqref{eq.lemma4-2} is exactly the RHS of \eqref{eq.trig-cond}, the trigger condition \eqref{eq.trig-cond} will not be activated, and the learner $m$ does not communicate with the controller at iteration $k$. 

Note that the above argument holds for any $1\leq d'\leq d$, and thus if \eqref{eq.lemma4} holds, the learner $m$ communicates with the controller at most every other $d$ iterations. Since \eqref{eq.lemma4-1} holds with probability $1-2\delta/K$, by using union bound, this argument holds with probability $1-2\delta$ for all $k\in\{1,\cdots, K\}$.

\subsection{Proof of Theorem \ref{theorem2}}\label{pf.theorem2}
Recalling the Lyapunov function \eqref{eq.Lyap}, we have
\begin{equation}
\small
\mathbb{V}^k:={\cal L}(\bbtheta^k)-{\cal L}(\bbtheta^*)+\sum_{d=1}^{D}\frac{3\sum_{j=d}^D\xi_j}{2\alpha}\left\|\bbtheta^{k+1-d}-\bbtheta^{k-d}\right\|^2
\end{equation}
Using \eqref{eq.pf-lemma2-6} in the proof of Theorem \ref{theorem1}, and choosing the stepsize as $\alpha=\frac{1}{L}\big(1-3D \xi\big)$, we have
\begin{align}
\small\mathbb{V}^{k+1}-\mathbb{V}^k&\small\leq -\frac{\alpha}{2}\left\|\nabla{\cal L}(\bbtheta^k)\right\|^2+\frac{3\alpha}{2}\Big\|\nabla_T {\cal L}(\bbtheta^k)-\nabla{\cal L}(\bbtheta^k)\Big\|^2\nonumber\\
& \small+\frac{3\alpha}{2}\Big\|\hat{\nabla}_{N,T}{\cal L}\big(\bbtheta^k\big)-\nabla_T {\cal L}(\bbtheta^k)\Big\|^2+9\alpha\sigma^2_{N,\delta/K}.
\end{align}
Summing up both sides from $k=1,\ldots,K$, and initializing $\bbtheta^{1-D}=\cdots=\bbtheta^0=\bbtheta^1$, we have
\begin{align}\label{eq.comp1}
\!\!\!\small\frac{1}{K}\sum_{k=1}^K	\left\|\nabla{\cal L}(\bbtheta^k)\right\|^2\!\!&\small\!\leq\!\frac{2 L\left[{\cal L}(\bbtheta^1)-{\cal L}(\bbtheta^*)\right]}{(1-3D \xi)K} \!+\!3\Big\|\nabla_T {\cal L}(\bbtheta^k)\!-\!\nabla{\cal L}(\bbtheta^k)\Big\|^2\nonumber\\
& \small+\!3\Big\|\hat{\nabla}_{N,T}{\cal L}\big(\bbtheta^k\big)\!-\!\nabla_T {\cal L}(\bbtheta^k)\Big\|^2\!\!+\!18\sigma^2_{N,\delta/K}\nonumber\\
& \small\leq \frac{2 L\left[{\cal L}(\bbtheta^1)-{\cal L}(\bbtheta^*)\right]}{(1-3D \xi)K}+3\sigma_T^2+21\sigma^2_{N,\delta/K}
\end{align}
where the last inequality holds w.p. $1-\delta$ by using the gradient concentration result in Lemma \ref{lemma.PGwhp}.

 With regard to PG, following the same line of analysis, it can guarantee that 
\begin{align}\label{eq.comp2}
&\small\frac{1}{K}\sum_{k=1}^K\left\|\nabla{\cal L}(\bbtheta^k)\right\|^2\nonumber\\
\leq &\small\frac{2 L}{K} \left[{\cal L}(\bbtheta^1)-{\cal L}(\bbtheta^*)\right]+3\Big\|\nabla_T {\cal L}(\bbtheta^k)\!-\!\nabla{\cal L}(\bbtheta^k)\Big\|^2\nonumber\\
&\small+3\Big\|\hat{\nabla}_{N,T}{\cal L}\big(\bbtheta^k\big)\!-\!\nabla_T {\cal L}(\bbtheta^k)\Big\|^2\nonumber\\
\leq &\small\frac{2 L}{K} \left[{\cal L}(\bbtheta^1)-{\cal L}(\bbtheta^*)\right]+3\sigma_T^2+3\sigma^2_{N,\delta/K},~{\rm w.p.}~1-\delta.
\end{align}
If $T$ and $N$ are chosen large enough (cf. \eqref{eq.theorem1-0}), so the first terms in the RHS of \eqref{eq.comp1} and \eqref{eq.comp2} each dominates the corresponding remaining two error terms. 
Therefore, to achieve the same $\epsilon$-gradient error, with probability $1-2\delta$, the number of needed iterations under LAPG is $(1-3D \xi)^{-1}$ times that of PG.

Regarding the number of needed communication rounds, similar to the derivations in \cite[Proposition 1]{chen2018nips}, we can use Lemma \ref{lemma4} to show that the LAPG's average communication rounds per iteration is $(1-\Delta\bar{\mathbb{C}}(h;\{\gamma_d\}))$ times that of PG with probability $1-2\delta$. Together with the number of needed iterations above, we arrive at \eqref{prononconvex-r1} with probability $1-4\delta$.

As $h(\cdot)$ is non-decreasing, for a given $\gamma'$, if $\gamma_D\geq \gamma'$, it readily follows that $h(\gamma_D)\geq h(\gamma')$.
Together with the definition of $\Delta\bar{\mathbb{C}}(h;\{\gamma_d\})$ in \eqref{eq.prop5}, we arrive at 
\begin{align}
\Delta\bar{\mathbb{C}}(h;\{\gamma_d\})&\small=\sum_{d=1}^D\left(\frac{1}{d}-\frac{1}{d+1}\right)h\left(\gamma_d\right)\nonumber\\
&\small\geq\sum_{d=1}^D\left(\frac{1}{d}-\frac{1}{d+1}\right)h\left(\gamma_D\right)\geq\frac{D}{D+1}h(\gamma').
\end{align}
Therefore, if we choose the parameters as 
\begin{equation}
\alpha=\frac{1-3D\xi}{L},~~\xi_d=\xi,~~\gamma_d=\frac{\xi/d}{3\alpha^2 L^2M^2},\, d\in[1,D]
\end{equation}
the total communication is reduced if the following relation is satisfied (cf. \eqref{prononconvex-r1})
\begin{align}\label{eq.67}
\frac{\mathbb{C}_{\rm LAPG}(\epsilon)}{\mathbb{C}_{\rm PG}(\epsilon)}=\Big(1-\frac{D}{D+1}h(\gamma')\Big)\cdot\frac{1}{1- 3D\xi}<1.
\end{align}
Clearly, \eqref{eq.67} holds if we have $h(\gamma')>3(D+1)\xi$.
On the other hand, the condition $\gamma_D\geq \gamma'$ requires
\begin{align}\label{dayu}
\xi/D\geq \gamma'(1-D\xi)^2M^2.
\end{align}
Clearly, if $\xi>\gamma' D M^2$, then \eqref{dayu} holds. In all, if we have
\begin{align}\label{xiao}
\gamma'<\frac{\xi}{D M^2}<\frac{h(\gamma')}{3(D+1)D M^2}
\end{align}
then the inequality $\mathbb{C}_{\rm LAPG}(\epsilon)\leq \mathbb{C}_{\rm PG}(\epsilon)$ in Theorem \ref{theorem2} holds with probability $1-4\delta$.

 \balance

\end{document}